\documentclass[10pt]{article} %
\usepackage[preprint]{tmlr}

\usepackage{amsmath,amsfonts,bm}

\def\eqref#1{equation~\ref{#1}}

\def\1{\bm{1}}

\DeclareMathAlphabet{\mathsfit}{\encodingdefault}{\sfdefault}{m}{sl}
\SetMathAlphabet{\mathsfit}{bold}{\encodingdefault}{\sfdefault}{bx}{n}

\newcommand{\E}{\mathbb{E}}

\newcommand{\Var}{\mathrm{Var}}

\usepackage{hyperref}
\usepackage{url}
\usepackage{graphicx}
\usepackage{booktabs}
\usepackage{amsmath,amssymb,amsthm}
\usepackage{microtype}
\usepackage{xcolor}
\usepackage{float}

\graphicspath{{figures/}}
\newtheorem{proposition}{Proposition}
\newcommand{\Rstar}{R^{*}}
\newcommand{\Rhat}{\hat R}
\newcommand{\gap}{\Delta}
\newcommand{\ks}{Kuramoto-Sivashinsky}
\newcommand{\ci}[3]{#1\,[#2,\,#3]}

\title{Spatial Context or Memory? Sizing the Inputs of Neural\\ Surrogates for Partially Observed Dynamical Fields}

\author{\name Ruige Kong \email rk684@uw.edu \\
      \addr University of Washington, Seattle, WA, USA}

\def\month{MM}  %
\def\year{YYYY} %
\def\openreview{\url{https://openreview.net/forum?id=XXXX}} %

\begin{document}
\raggedbottom

\maketitle

\begin{abstract}
Neural surrogates that forecast a physical field from sparse, noisy sensors read a window of nearby sensors and a history of past snapshots. Practitioners choose both by comparing fitted models with each other. Such comparisons cannot tell a model's shortfall from missing information. We measure the tradeoff between window and history against the Bayes-optimal error of the same readings. An exact Gaussian analysis shows that noisy or sparse observation of a field with local dynamics makes the optimal forecast read distant sensors only through spatial correlation. On a forced linear field with a known optimum, trained models that read only the window come within 1\% of it. We then train such models on viscous Burgers flow and the chaotic \ks{} (KS) equation, with the deciding comparisons fixed before any test evaluation. In the fields other than KS, eight snapshots remove a smaller fraction of the one-snapshot error as the lead time, the number of steps to the forecast target, grows. In KS, for windows up to one correlation length, they remove a larger fraction. This lead-time inversion persists across sensor densities, three architectures and a two dimensional domain. It shrinks as damping lowers the Lyapunov exponent. Damping also removes the linear instability of KS. We therefore tie the inversion to instability and leave chaos and linear amplification unseparated. For surrogate design the results give two rules. Size the window in correlation lengths of the field. Split a fixed input budget between window and history for the lead time the model will serve, because the best split moves toward wider windows as the lead time grows. In Burgers flow the split chosen at lead time one ends with 13.5 times the best error at lead time eight.

\end{abstract}

\section{Introduction}
\label{sec:intro}

A neural surrogate of a sparsely observed field must decide how much spatial context and how much memory to read. Many physical fields, such as a temperature, a velocity or a concentration, obey local dynamics. The rate of change at each point depends only on the field in a small neighbourhood of that point, as in diffusion, advection and most partial differential equations of physics. The same choice arises for learned estimators and forecasters of partially observed dynamical systems, because they must infer the hidden local state from the readings they are given. Sensors observe such fields at sparse points and with noise. A surrogate that forecasts the field at a query point, the location where a forecast is wanted, reads a window of sensors around that point and a history of past snapshots of those sensors. Local forecasters with memory \citep{pathak2018parallel}, local neural operators \citep{ye2024locality} and global operators \citep{li2021fno} fix these two choices differently. Recent work asks how much memory an operator needs \citep{buitrago2025memory, hur2026memory}. These studies vary one choice at a time and compare models only with each other. Without the error of the best predictor of the same inputs, they cannot tell a model's shortfall from a lack of information. We supply that reference. Where the field is Gaussian we compute the optimal error in closed form. In the nonlinear fields we use trained models, after checking on a linear field with a known optimum that they reach it.

A natural argument holds that noisy or coarse observation forces the optimal forecast to read distant sensors. We show that this holds only through spatial correlation. Consider a Markov field, whose next state depends only on its current state, observed through noisy or coarse sensors. We call such observation degraded. Degraded readings are in general not Markov in time, because the current readings do not determine the current state. From this, an argument we call the inference hypothesis concludes that the Bayes-optimal forecast of the next value at a point must infer the hidden state from distant sensors, even though the dynamics are local. Exact Gaussian conditioning gives a counterexample (Section~\ref{sec:exact}). With one snapshot of a state that is uncorrelated across grid points, the optimal predictor reads exactly the grid points that one step of the dynamics links to the target, at every noise level. The counterexample is exact. So the step from degraded observation to distant sensors fails in general. Typical fields still satisfy the conclusion, because stationary local dynamics produce correlated states. A correlated state, or a history of snapshots, does extend the predictor to distant sensors.

History can replace a wider window only to the extent that the window's own past carries the information that distant sensors would supply. We call the extra error of a window, compared with a window that holds every sensor, its spatial gap. We say that history substitutes for space when extra past snapshots of a window close part of its spatial gap. On this claim substitution can fail in two ways. In the first failure mode the field has spatial structure between the window's sensors that the window's own past does not carry either. In the second, transport carries the needed information into the window from outside faster than history can supply it. We observe both failure modes. Each depends on the field and its noise level (Section~\ref{sec:context}).

In the chaotic \ks{} (KS) equation, history behaves opposite to the other fields we study. We study three fields on a ring, a randomly forced linear field whose optimal predictor we can compute in closed form, viscous Burgers flow and KS. The lead time is the number of observation steps between the latest snapshot and the forecast target. In KS, for windows up to one correlation length, eight snapshots remove a larger fraction of the one-snapshot error at longer lead times. We call this the lead-time inversion. Burgers flow and the linear field show the reverse at the windows we tested, as do the two linear chains we analyse exactly in Section~\ref{sec:exact}. We had predicted the same for KS. A controlled experiment argues against four causes of the inversion, namely the window's size relative to the correlation length, structure finer than the sensor spacing, the sensor count and the model architecture. When we add damping to KS, the inversion shrinks as the largest Lyapunov exponent, the growth rate of small perturbations, falls. It disappears when the exponent is negative. This trend ties the inversion to instability. It cannot separate chaos from linear instability, the exponential growth of small long-wave perturbations, because damping removes both.

The inversion holds on a second geometry and changes how a surrogate should be configured. On a periodic square it recurs in KS up to one correlation length and stays absent on a linear field with a known optimum. When a fixed number of readings per query is split between window and history, splits with long histories lose ground to wide windows as the lead time grows, in each of the three fields. Two safeguards support these findings. Before any test evaluation we fixed in writing which comparisons decide each experiment and which direction we predict. On the linear field the trained models come within 1\% of the exact optimum of their inputs. Throughout, we keep exact quantities apart from the errors of trained models.\footnote{The repository at \url{https://anonymous.4open.science/r/memory-space-forecasting-5633} holds the code, configurations, analysis plans and the scripts behind the figures and tables.}

For building surrogates the results give a method and two rules. The method scores a trained model against the exact optimum of its own inputs wherever one exists. It separates a model's shortfall from missing information, which comparisons among fitted models cannot do. One rule sizes the window in correlation lengths of the field. In KS the lead-time behaviour of history changes between one and two correlation lengths, at two sensor densities and on two geometries. A window of four sensor spacings inverts with 256 sensors. With 64 sensors its history ratio rises. The other rule, for sensor design, chooses the split between window and history at the lead time the model will serve. A split tuned at lead time one carries too much history at lead time eight. In Burgers flow it ends with 13.5 times the best error.

\section{Setting and definitions}
\label{sec:setting}

\begin{figure}[t]
\centering
\includegraphics[width=\textwidth]{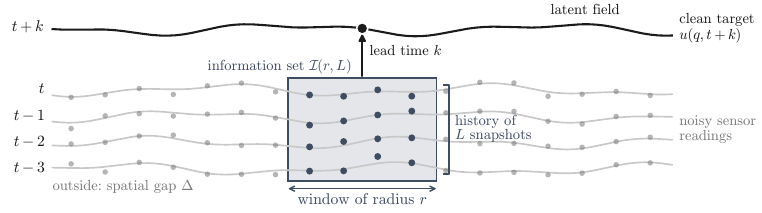}
\caption{The forecasting problem. Sensors read a latent field on a periodic domain with additive noise. A local predictor for the query point reads the sensors inside a window of radius $r$ (dark dots in the shaded box) over the $L$ most recent snapshots. It predicts the clean field value $k$ observation steps ahead (black dot). The grey sensors outside the window carry the spatial gap $\gap$. The exact predictor from these readings has risk $\Rstar(r,L,k)$. A trained model has empirical risk $\Rhat(r,L,k)$.}
\label{fig:setting}
\end{figure}

The comparisons in this paper hold the forecast target fixed and change only the readings a predictor sees (Figure~\ref{fig:setting}). A latent field $u(x,t)$ evolves on a periodic domain of length $D$ with $n$ grid points, recorded at times $t_j = j\,\delta t$. Sensors at locations $\mathcal S$ return $y_j(s) = u(s,t_j) + \varepsilon_j(s)$, with independent Gaussian noise of standard deviation $\sigma$. For a query point $q$, the \emph{window} of radius $r$ contains the sensors within periodic distance $r$ of $q$. The \emph{history} $L$ is the number of most recent snapshots the predictor reads, which we also call memory. The \emph{lead time} $k$ is the number of observation steps to the target, the clean value $u(q, t_{j+k})$. Together these define the information set $\mathcal I(r,L)$, the readings in the window over the history. We score predictors by their mean squared error on the target, divided by the target variance. Each comparison pairs information sets on identical targets, queries, lead time, observation interval and noise realization. A difference between two exact predictors is then a difference of information alone.

The spatial gap isolates the information that lies outside the window. The \emph{exact risk} $\Rstar(r,L,k)$ is the Bayes risk of $\mathcal I(r,L)$, in closed form for Gaussian fields. A trained model's test error is its \emph{empirical risk} $\Rhat(r,L,k)$. Because the information sets are nested in $r$ and in $L$, $\Rstar$ cannot increase in either. The \emph{spatial gap}
\begin{equation}
\gap(r,L,k) = \Rstar(r,L,k) - \Rstar(\Omega,L,k) = \E\big[(m_{\Omega,L} - m_{r,L})^2\big] \ge 0
\label{eq:gap}
\end{equation}
measures how much the sensors outside the window reduce the exact risk at the same history. Here $\Omega$ is the window that holds every sensor and $m_{r,L}$ is the conditional mean given $\mathcal I(r,L)$. The same difference between the errors of two trained models, $\hat\gap$, mixes information with fit and can be negative. For the nonlinear fields we estimate $\Rstar(\Omega,L,k)$ by the error of the best global model, a model that reads every sensor. Gaps in those fields are therefore estimates (Section~\ref{sec:protocol}).

Three quantities measure what history adds. They answer different questions. The \emph{gap ratio} $g(r,k) = \gap(r,8,k)/\gap(r,1,k)$ is the fraction of the spatial gap that eight snapshots leave open. A smaller gap ratio means stronger substitution of history for space. The \emph{history ratio} $\rho(r,k) = R(r,8,k)/R(r,1,k)$, with $R$ exact or empirical as stated, is the fraction of the one-snapshot error that remains with eight snapshots. In absolute units, the \emph{error reduction from history} is $B(r,k) = R(r,1,k) - R(r,8,k)$. Its ratio to the one-snapshot gap, $B/\gap(r,1,k)$, is the \emph{closed fraction}. The gap ratio measures substitution itself. The history ratio needs no global reference. Most of the comparisons we fixed in advance use it, because its estimates carry no error from a reference. A \emph{lead-time inversion} occurs when $\rho(r,8) < \rho(r,1)$, when eight snapshots remove a larger fraction of the one-snapshot error at lead time eight than at lead time one. The \emph{sufficiency radius} $r_\epsilon(L,k)$ is the smallest window whose spatial gap is at most $\epsilon$.

A window has two sizes that can disagree. We report both. The sensor spacing gives the number of sensors a query sees. The correlation length $\ell$, the distance at which the spatial autocorrelation of the field falls to $1/e$, gives the amount of the field's structure that the window contains. A comparison across fields that matches one size leaves the other free, which Section~\ref{sec:leadtime} exploits.

\section{Exact analysis in a linear Gaussian chain}
\label{sec:exact}

A linear Gaussian chain gives every quantity of Section~\ref{sec:setting} in closed form, which makes it an exact reference for the trained models. Let the state $x_j \in \mathbb R^n$ evolve as $x_{j+1} = A x_j + w_j$, where $w_j$ is Gaussian noise. We call the entries of $x_j$ sites, the grid points of a discretized field. The update $A$ is local. Each site's new value depends only on nearby sites. Let the sensors return $y_j = H x_j + v_j$ with $v_j \sim \mathcal N(0, \sigma^2 I)$. We call the sites that one application of $A$ couples to a site its stencil. If the current state has covariance $P$, the conditional mean of the target is linear, $\E[x_{j+k} \mid y_j] = M y_j$ with gain $M = A^k P H^\top (H P H^\top + \sigma^2 I)^{-1}$. The risk of any information set is a conditional variance \citep{kalman1960}. The dynamics enter the cross covariance on the left. A gain with $A^\top$ in that place is correct only for symmetric $A$. For that reason the sweeps below include an advective update, which shifts the field and is not symmetric, as well as a diffusive one, which smooths it. We use no regularization (Appendix~\ref{app:gaussian}).

\begin{proposition}[One degraded snapshot keeps the one-step predictor inside the stencil]
\label{prop:independent}
Let the sensor matrix have orthogonal rows of equal squared norm, $H H^\top = c I$. A selection $H = S$ of any subset of sites has $S S^\top = I$ and $c = 1$. A non-overlapping block average of width $w$ has $c = 1/w$. Let the current state be uncorrelated across sites, $P = pI$ with $p > 0$. Then for every $\sigma \ge 0$,
\[
\E[x_{j+1} \mid y_j] = \frac{p}{c\,p + \sigma^2}\, A H^\top y_j .
\]
For a selection $H = S$ the same conclusion holds with any diagonal $P$. The gain is then $A P S^\top (S P S^\top + \sigma^2 I)^{-1}$, the stencil times a diagonal matrix. In both cases the optimal one-step predictor from one snapshot reads only sensors inside the stencil of the target. Its gain shrinks toward zero as the noise grows.
\end{proposition}
\begin{proof}
The cross covariance is $A P H^\top = p A H^\top$. With $H H^\top = cI$ the observation covariance is $H P H^\top + \sigma^2 I = (c\,p+\sigma^2) I$. Substituting both into the gain gives the first result. With $H = S$ and diagonal $P$, the matrix $S P S^\top$ is diagonal. The inverse $(S P S^\top + \sigma^2 I)^{-1}$ is then diagonal. The product $A P S^\top$ therefore has the sparsity of $A S^\top$.
\end{proof}

Noise and subsampling alone leave a one-snapshot predictor inside the stencil. With every site observed and an uncorrelated state, the spatial gap beyond the stencil is below $2 \times 10^{-16}$ of the target variance. It stays at this level at all eighteen noise levels we tried, from $0$ to $100$ field units (Figure~\ref{fig:gaussian}a, grey squares). The computation agrees with the proposition to machine precision. The counterexample to the inference hypothesis is therefore exact. It covers one snapshot of an uncorrelated state. The next paragraph gives the cost of each assumption.

Spatial correlation is the only route by which degraded observation reaches distant sensors. Typical fields provide it. First, the proposition covers one snapshot. Over a history of $L$ snapshots the dynamics correlate a state that starts uncorrelated. The spatial gap beyond one site then reaches $2.9\%$ of the target variance at $L = 8$, although it falls to $2.5 \times 10^{-7}$ beyond five sites (Appendix~\ref{app:gaussian}). Second, the premise $P = pI$ is atypical for local stochastic dynamics. A stationary state solves $P = A P A^\top + Q$, which is correlated whenever the stencil couples neighbours. So degraded observation gives distant sensors weight in the optimal predictor only through correlation, whether the correlation sits in the current state or builds up over the history. The proposition refutes the inference hypothesis in general. It leaves the hypothesis open in correlated fields, which the rest of the paper studies.

\begin{figure}[t]
\centering
\includegraphics[width=\textwidth]{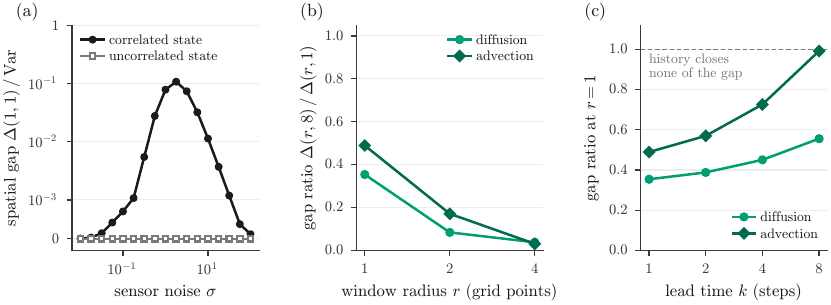}
\caption{Exact results in the linear Gaussian chain (32 sites, nearest-neighbour update, Appendix~\ref{app:gaussian}). (a) The spatial gap beyond one grid point, $\gap(r{=}1, L{=}1)/\Var$, against the sensor noise $\sigma$. The log axis omits $\sigma = 0$, where the gap of the uncorrelated state is also zero. With an uncorrelated current state the gap is zero at every noise level (Proposition~\ref{prop:independent}). With the correlated state it peaks and then falls. (b) The gap ratio $\gap(r,8)/\gap(r,1)$ at the peak noise level, for the diffusive and the advective update. It is shown at the three windows whose one-snapshot gap exceeds $10^{-4}$ of the target variance. A lower gap ratio means that eight snapshots close more of the spatial gap. (c) The same ratio at $r=1$ against the lead time $k$. Under advection it approaches one. History then closes almost none of the spatial gap.}
\label{fig:gaussian}
\end{figure}

With a correlated state, distant sensors matter most at moderate noise. We start from a smooth, spatially correlated covariance and propagate it eight steps under the diffusive update. The spatial gap beyond one grid point then rises with noise to a peak of 11\% of the target variance near $\sigma \approx 1.8$. Beyond the peak it falls about 900-fold by $\sigma = 100$ (Figure~\ref{fig:gaussian}a, black). Because noise shrinks the whole gain, the predictor looks wider only once its weights are rescaled to a fixed total. Even then it widens only up to the peak. Strong noise therefore narrows the optimal predictor again. More noise does not mean more reliance on distant sensors, even in a correlated field.

History substitutes for space better as the window widens. At the peak noise level the gap ratio falls from $0.35$ at $r = 1$ to $0.084$ at $r = 2$ and $0.038$ at $r = 4$ (Figure~\ref{fig:gaussian}b). The sufficiency radius at $\epsilon = 10^{-3}$ shrinks from eight grid points with one snapshot to four with two or more. The ordering is plausible, because the past of a wider window determines the state at its own edge, which is the information the outside sensors would otherwise supply. It is exact for the linear chain. Section~\ref{sec:context} tests whether it holds in the nonlinear fields.

History closes less of the spatial gap as the lead time grows. Under transport almost none of the effect remains. At $r = 1$ the gap ratio rises from $0.35$ at $k = 1$ to $0.56$ at $k = 8$ under diffusion (Figure~\ref{fig:gaussian}c). Under the advective update, which moves the field $0.1$ sites per step, it rises from $0.49$ to $0.99$. Although the field drifts only $0.8$ sites over eight steps, the forecast then depends on field arriving from upstream, which the window did not observe. The transport limit is therefore severe even for slow drift. Both orderings, with window and with lead time, are exact for these two chains. We treat them as predictions for the nonlinear fields.

\section{Fields, models and protocol}
\label{sec:protocol}

\begin{figure}[t]
\centering
\includegraphics[width=\textwidth]{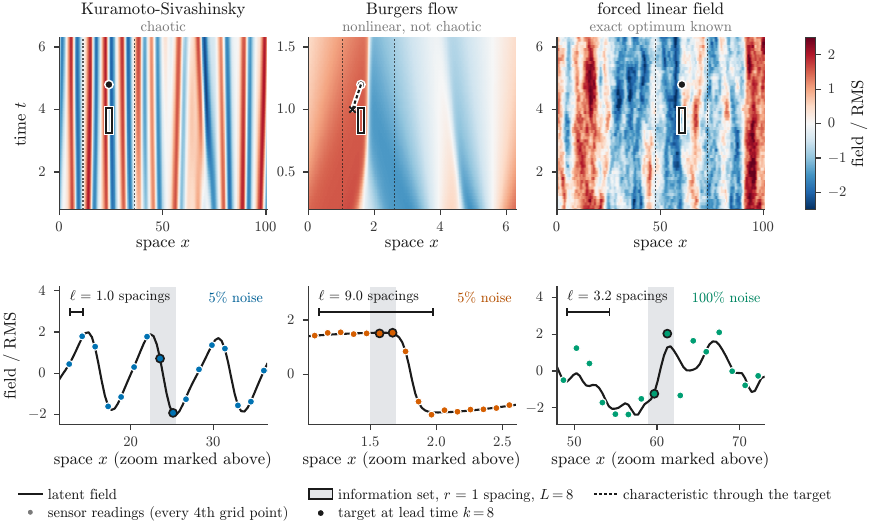}
\caption{The three fields and their sensor arrays. Top: one training trajectory in space and time. Each panel marks one information set, a window of one sensor spacing over eight snapshots (outlined box). The black dot is its target eight steps ahead. In Burgers flow the dotted curve follows the characteristic through the target back to the current time. Its foot (cross) lies upstream of the window, which is the transport failure of Section~\ref{sec:context}. Dashed lines mark the region enlarged below. Bottom: the snapshot at the current time with the readings a model receives. The shaded band is the window of the information set. Outlined dots are the sensors inside it. A bar draws the correlation length $\ell$ to scale. The correlation length is one sensor spacing in KS, nine in Burgers flow and three in the linear field. The readings on the linear field are far noisier than on the other two fields.}
\label{fig:fields}
\end{figure}

The three fields separate calibration from the effect of nonlinear dynamics (Figure~\ref{fig:fields}). The randomly forced linear field is advection-diffusion driven by Gaussian forcing, an exact Gaussian chain between observation times, which gives $\Rstar(r,L,k)$ for every information set. Its diffusive variant has no drift. Its advective variant drifts at constant speed. The two nonlinear fields, viscous Burgers flow and the chaotic \ks{} (KS) equation \citep{kuramoto1978, sivashinsky1977}, share the quadratic term $u\,\partial_x u$. They differ in what they do to small perturbations. KS amplifies small perturbations at every wavelength longer than $2\pi$, which we call its linear instability. It is also chaotic, with a largest Lyapunov exponent of $0.089$ per time unit on our domain. Burgers flow damps small perturbations. A contrast between the two nonlinear fields therefore points to instability. It cannot isolate instability, because Burgers flow also decays from a random initial state, does not become stationary and forms shocks.

We match windows across fields by the number of sensors a query sees. This choice leaves the window's size in correlation lengths free. The fields differ in that size. We integrate the nonlinear fields with a fourth-order exponential time-differencing scheme on 256 grid points \citep{kassam2005etdrk4}. Unless stated otherwise, point sensors sit at every fourth grid point, with noise of 5\% of the field's training root mean square on the nonlinear fields and 100\% on the linear field. Appendix~\ref{app:systems} gives the settings and checks. The sensor spacing is then one correlation length in KS, a ninth of one in Burgers flow and a third of one in the linear field. Section~\ref{sec:leadtime} uses this difference to separate two explanations of its main result.

Each local model sees its own information set and nothing more. A local model reads only the sensors in its window. We call such a model a patch model. Each multiplies its input by a fixed mask before any learned operation. The patch CNN applies three convolutions of width three to the masked window and averages the result over the window. In one experiment a patch GRU and a patch MLP replace it. The linear patch is the empirical counterpart of the Gaussian gain. For the global reference we use a Fourier neural operator \citep{li2021fno} that reads every sensor over the same history and keeps every mode the array resolves (Appendix~\ref{app:models}). All models train with AdamW and cosine decay over 150 epochs. Each combination of window, history and lead time has five training seeds, or three on the linear fields. Because of the mask, a difference between two windows comes from information and fit. No model reads inputs outside its window.

We fixed the deciding comparisons before seeing test data. Before training any model of an experiment, we wrote down the comparisons that decide its outcome, the direction we predicted for each and the interval level. We call these the planned comparisons. Each analysis ran first on the validation split. We then recorded a content hash of the written plan and evaluated the test split once, with the same code. The numbers below come from the test split. Intervals come from a crossed bootstrap that resamples training seeds and test trajectories in 2000 draws and keeps each comparison paired (Appendix~\ref{app:stats}). Within each experiment a Bonferroni correction covers its planned comparisons, which gives 99.17\% intervals for six comparisons and 98.75\% for four. All other quantities carry 95\% intervals. This procedure keeps us from choosing comparisons after seeing their results. Two exceptions weaken it. We wrote the KS predictions after a single-seed pilot study scored on validation. We also designed the experiment of Section~\ref{sec:leadtime} that separates two explanations after the KS and Burgers test results were known. It reuses their test trajectories. Appendix~\ref{app:plans} lists all fifty planned comparisons, including the ten that did not come out as predicted.

Gaps in the nonlinear fields are estimates against the best global model. In both nonlinear fields the Fourier neural operator is the best of our models at each history and lead time. We therefore quote gaps against it. Such gaps understate the spatial gap where the operator falls short of the optimum and overstate it where the local model falls short. A reference that fits poorly would inflate the gaps, since weak baselines are a documented source of overoptimism \citep{mcgreivy2024weak}. On the linear field the operator fits the heavy sensor noise and sits 2\% to 3\% above the optimum. There we compare with the exact optimum directly.

On the linear field the trained models reach the exact optimum of their inputs. We compared each fitted error there with the exact predictor on the same test windows, noise realizations and queries. Across all thirty combinations of window, history and lead time in each variant, the linear patch sits $0.01\%$ to $0.35\%$ above its exact optimum. For the patch CNN the excess is $0.05\%$ to $0.65\%$ (Appendix~\ref{app:calib}). Its history ratios also match the exact ones to within $0.002$ at all sixteen combinations of window and lead time. The models therefore recover the structure as well as the level of the risk. This calibration is strong for linear structure. Because the optimum on this field is linear, it cannot show that the CNN captures nonlinear interactions across a wide window (Section~\ref{sec:limitations}).

\section{Results}
\label{sec:results}

\subsection{Window and history in two nonlinear fields}
\label{sec:context}

\begin{figure}[t]
\centering
\includegraphics[width=\textwidth]{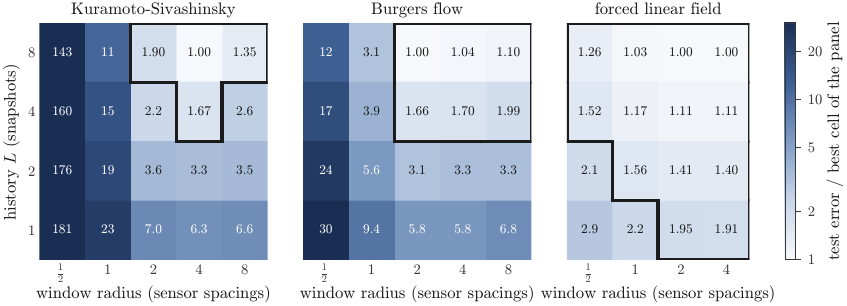}
\caption{Trading window for history at fixed error, on the test split at lead time one. Each cell is one trained patch CNN. The number in the cell and its shade give the cell's test error relative to the best cell of its panel. The heavy outline encloses the cells within twice that best error. An outline with a diagonal edge means that a longer history compensates for a narrower window. At one spacing the excess is eleven times the best error in KS and three times in Burgers flow.}
\label{fig:tradeoff}
\end{figure}

In both nonlinear fields history removes more of the error at wider windows, as we predicted. In KS with 64 sensors the history ratio falls from $0.79$ with one or two sensors to $0.49$, $0.27$ and $0.16$ at the next three windows. Burgers flow follows the same order, with ratios of $0.41$, $0.33$ and $0.17$. Each planned difference excludes zero at its corrected level, also with 32 KS sensors (Appendix~\ref{app:plans}). The ordering is therefore firm at short lead times. History cannot rescue the narrowest windows, however (Figure~\ref{fig:tradeoff}, error curves in Appendix~\ref{app:surface}). In both nonlinear fields no history we tried brings a window narrower than two sensor spacings within twice the best error. On the linear field, by contrast, a longer history compensates for a narrower window along a diagonal edge.

Substitution itself, measured by the gap ratio, follows the exact ordering only for small windows at short lead times. In KS at lead time eight the gap ratio rises from $0.32$ at one correlation length to $0.91$ at four. These ratios divide two fitted gaps and come from an analysis added after the test evaluation. We treat them as descriptive (Appendix~\ref{app:contrasts}).

The spatial gap grows sharply as the array thins. With 64 KS sensors a window of eight or nine sensors sits $\ci{1.7}{1.4}{2.0} \times 10^{-4}$ above the global operator at one snapshot. With 32 sensors the local models fall short by $\ci{0.054}{0.046}{0.063}$, an upper estimate of the information gap (Appendix~\ref{app:suff}).

The first failure mode, structure between sensors that the window's past does not carry, appears in KS and fades in the smoother or noisier fields. For queries that see exactly one sensor, eight snapshots of that sensor close $\ci{0.195}{0.183}{0.208}$ of the one-snapshot spatial gap in KS. In Burgers flow they close $\ci{0.75}{0.68}{0.94}$. On the diffusive and advective linear fields they close $\ci{1.50}{1.46}{1.55}$ and $\ci{0.92}{0.91}{0.93}$, where the exact predictor gives $1.51$ and $0.93$. A value above one means that eight snapshots of one sensor beat one snapshot of the whole array. The intervals do not overlap. The fields therefore differ. We attribute the spread to what the sensors miss. In KS that is structure between sensors, which one sensor's past barely carries. In Burgers flow, smooth over nine spacings, it is mostly noise and the position of the occasional shock. On the linear field it is noise. The fields also differ in noise level, 5\% against 100\%. This attribution therefore remains one of two possible explanations.

Transport produces the second failure mode once the field moves within the lead time. In Burgers flow the mean displacement over eight steps is 1.5 sensor spacings. At the window of half a sensor spacing the closed fraction falls from $0.81$ at lead time one to $0.56$ at lead time eight, a difference we predicted in advance, of $\ci{-0.25}{-0.52}{-0.17}$. Against the global operator the fall is largest at two spacings, about the displacement, from $2.51$ to $0.53$. Its interval is wide but excludes zero. This location of the largest fall matches what transport predicts. Our evidence rests on one drift speed in one field (Section~\ref{sec:limitations}).

\subsection{The lead-time inversion}
\label{sec:leadtime}

\begin{figure}[t]
\centering
\includegraphics[width=\textwidth]{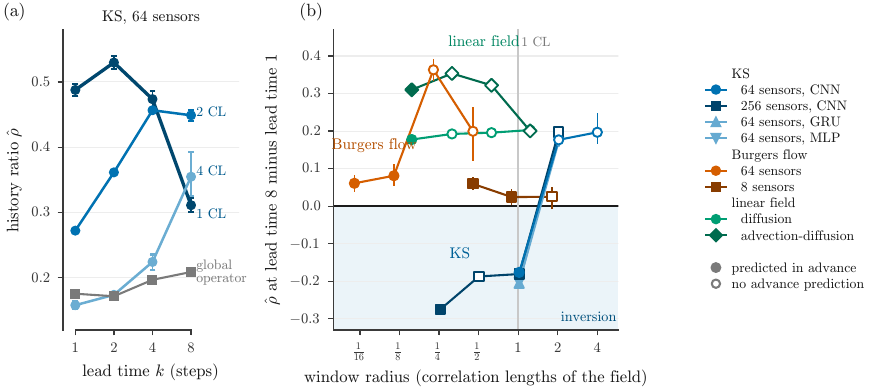}
\caption{The lead-time inversion (test split). (a) The history ratio $\hat\rho = \hat R(r,8)/\hat R(r,1)$ in KS with 64 sensors at lead times one, two, four and eight, for windows of one, two and four correlation lengths (CL) and the global operator, with 95\% intervals. A lower ratio means that eight snapshots remove a larger fraction of the one-snapshot error. (b) The same comparison as a difference, $\hat\rho$ at lead time eight minus $\hat\rho$ at lead time one, for each window of each field, with 99.17\% intervals. The horizontal axis gives the window radius in correlation lengths of each field. Colour gives the field. Marker shape gives the array or the local model. Filled markers are comparisons whose sign we predicted in advance. Open markers carry no advance prediction. Points below zero are inversions. Windows spanning the whole ring, which are limited by optimization, are omitted. Table~\ref{tab:leadtime} lists the values.}
\label{fig:inversion}
\end{figure}

Only KS inverts the lead-time ordering, at windows up to one correlation length (Figure~\ref{fig:inversion}). The exact analysis predicts that history removes a smaller fraction of the error as the lead time grows. Both linear variants and Burgers flow show this at each window we tested. The Burgers history ratio rises by $\ci{+0.061}{+0.039}{+0.082}$ at the window of half a sensor spacing and by $\ci{+0.081}{+0.055}{+0.112}$ at one spacing. In KS, at one sensor spacing, which is also one correlation length, the ratio falls from $0.488$ to $0.311$, a difference of $\ci{-0.177}{-0.186}{-0.167}$. It rises, as we had predicted for all windows, only at two and four correlation lengths. Each lead-time comparison keeps its sign when we leave out training seeds one at a time. The inversion is therefore large relative to its interval and stable across seeds. Lead times two and four show how it develops (Figure~\ref{fig:inversion}a). At one correlation length the ratio first rises, from $0.488$ to $0.530$ at lead time two, as in the other fields. Only then does it fall, to $0.474$ and $0.311$. The inversion builds over several observation steps.

The inversion arises because history's error reduction grows faster than the error itself in KS. The error reduction $B$ grows with lead time at small windows in both nonlinear fields, because the errors themselves grow. In KS at one correlation length $B$ grows 5.2-fold from lead time one to eight, against 3.9-fold for the one-snapshot error. In Burgers flow at half a sensor spacing it grows 6.7-fold, against 7.4-fold (Appendix~\ref{app:contrasts}). History therefore removes more error at longer lead times in both fields. Only the balance between the two growth rates separates them. The inversion therefore concerns proportions.

The inversion could come from instability or from resolution. The two explanations predict opposite outcomes in two further arrays. Under instability, KS amplifies small perturbations, which Burgers flow damps. An error in the window's current state then grows over a long forecast. History reduces that error. Under resolution, the inverting window matters because it is the only one in these experiments whose radius equals the correlation length of its field. An inversion there could need spatial structure at that scale, with no role for instability. To tell the two apart, we ran three variants and fixed in advance how each outcome would be read. The first observes KS at every grid point, a quarter of a correlation length apart, which leaves nothing unresolved. The second observes Burgers flow through eight sensors, 0.89 correlation lengths apart, with 4.8\% of the variance beyond the array. If instability drives the inversion, densely observed KS should still invert and Burgers flow should not. If resolution drives it, Burgers flow should invert at a window near its correlation length. KS should then stop inverting at windows well inside its own. The third keeps the KS array and replaces the patch CNN with the patch GRU and the patch MLP.

The outcome favours instability over four alternative causes. Each of the six planned comparisons excludes zero with the sign the instability explanation predicts (Table~\ref{tab:decisive}). Densely observed KS inverts at a quarter, a half and one correlation length, most strongly at the smallest ($\ci{-0.275}{-0.281}{-0.270}$). At two its ratio rises again. Burgers flow through eight sensors shows a rising ratio at both windows. Across the two Burgers arrays no window from $0.06$ to $3.6$ correlation lengths inverts. The GRU and the MLP reproduce both signs, with one-snapshot errors that match the CNN's to two digits. So the experiment argues against the window's size relative to the correlation length, unresolved structure, the sensor count and the architecture as causes. The interpretation we fixed in advance reads the outcome as evidence that chaos is necessary and, inside one correlation length, sufficient. That reading claims more than the design supports, because chaos changes here only together with the equation. KS and Burgers flow also differ in stationarity, shocks, the sign of the second-derivative term, the observation interval and the ratio of noise to local structure.

\begin{figure}[t]
\centering
\includegraphics[width=\textwidth]{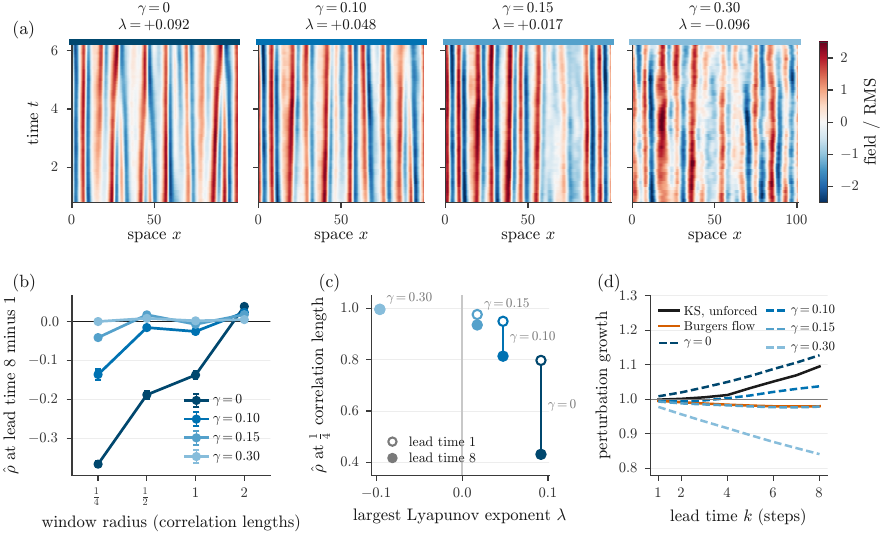}
\caption{The damped and forced KS family (test split). (a) One training trajectory of each damped field in space and time, with the damping $\gamma$ and the largest Lyapunov exponent $\lambda$. The coloured bar above each panel gives the field's colour in the other panels, darker for more chaotic fields. (b) History ratio at lead time eight minus lead time one against window radius, for the four damping levels, with 95\% intervals. (c) The history ratio at a quarter of a correlation length at lead times one and eight, against each field's largest Lyapunov exponent. (d) Median growth of a small perturbation one correlation length wide over the lead times of the experiments, on 96 states of each training set. Over eight steps the perturbation grows by at most 13\%.}
\label{fig:damping}
\end{figure}

Within one equation, the inversion shrinks steadily as damping lowers the Lyapunov exponent. We add uniform linear damping $\gamma$ to KS, with the same random forcing on the sixteen longest wavelengths to keep the field in motion (Figure~\ref{fig:damping}a). Four damping levels, $\gamma = 0$, $0.10$, $0.15$ and $0.30$, give largest Lyapunov exponents of $+0.092$, $+0.048$, $+0.017$ and $-0.096$, with correlation lengths between $1.62$ and $1.89$. We observe every grid point (Appendix~\ref{app:systems}). At a quarter of a correlation length the lead-time difference is $\ci{-0.366}{-0.376}{-0.356}$, $-0.136$, $-0.041$ and $\ci{+0.0004}{-0.0007}{+0.0015}$ from the most chaotic field to the least, with 98.75\% intervals (Figure~\ref{fig:damping}b). This order holds in all 2000 bootstrap draws. The trend with the exponent is therefore stable under resampling. Only our prediction for the non-chaotic field, a difference above zero, is unresolved.

The non-chaotic field cannot show whether the inversion needs chaos, because history barely reduces its error. Its history ratios are $0.996$ and $0.997$ at lead times one and eight (Figure~\ref{fig:damping}c). The strongly damped field is nearly Markov in the window's readings. One snapshot already determines the window's state. Its remaining error is new forcing, which no history can predict. With stronger damping the error reduction at lead time one also falls, from a history ratio of $0.80$ to $0.95$ and $0.98$. The trend with damping therefore ties the inversion both to the exponent and to how much the window's past reduces the error at all. Damping also removes the band of wavelengths that KS amplifies, whose peak growth rate falls from $0.25$ to $-0.05$ across the four fields. For this reason the experiment cannot separate chaos from linear instability.

The perturbation growth over our lead times is too small to explain the inversion by itself. Over lead time eight, only $0.071$ Lyapunov times of KS, a small perturbation one correlation length wide grows by a median of 10\% in KS and shrinks by 2\% in Burgers flow (Figure~\ref{fig:damping}d). The inversion appears where the dynamics amplify small perturbations, tracks the rate of that amplification and is absent wherever they decay. We therefore claim an association with instability and leave its mechanism open.

\subsection{The inversion in two dimensions and a design consequence}
\label{sec:twod}

\begin{figure}[t]
\centering
\includegraphics[width=\textwidth]{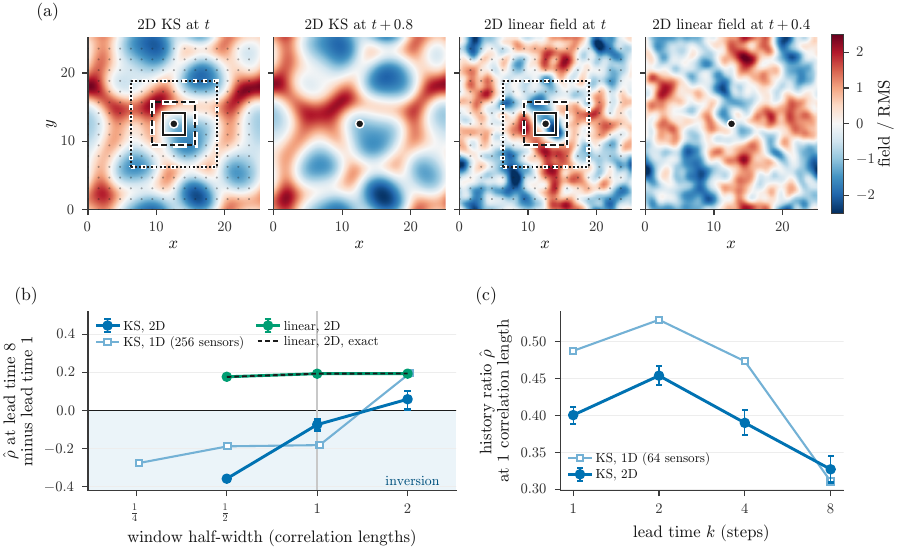}
\caption{The window and history tradeoff on the periodic square (test split). (a) The two dimensional KS equation and the two dimensional forced linear field at the current time $t$ and eight observation steps later. Grey dots are the sensors. The three squares are windows of half-width a half, one and two correlation lengths around one query (black dot), which is also the point the forecast targets. KS changes its pattern within the lead time. (b) History ratio at lead time eight minus lead time one against window half-width, with 99.17\% intervals. Points below zero are inversions. Pale squares give KS in one dimension, the dashed line the exact linear value. (c) The KS history ratio at one correlation length against lead time, with 95\% intervals, in one and two dimensions.}
\label{fig:2d}
\end{figure}

A ring could favour the inversion, because a window's edge there is two points and transport has one direction. We therefore repeated the experiment on the periodic square of side $8\pi$ with $64 \times 64$ grid points (Figure~\ref{fig:2d}a, Appendix~\ref{app:twod}). Two dimensional KS is chaotic there, with a largest Lyapunov exponent of $0.18$. A forced linear field with a known optimum serves as control, with both fields set to a correlation length of $\pi$ so that their windows contain comparable structure. We fixed seven predictions before either data set existed.

On the square KS inverts up to one correlation length and the linear field does not invert. Each of our seven predictions came out as predicted on validation and on test. Two dimensional KS inverts at half a correlation length, $\ci{-0.357}{-0.371}{-0.345}$. It inverts again at one, $\ci{-0.073}{-0.104}{-0.046}$ (Figure~\ref{fig:2d}b). At two the ratio rises again, $\ci{+0.060}{+0.006}{+0.105}$. No window of the linear field inverts. Its fitted lead-time differences match the exact ones to within $0.001$. At one correlation length the KS history ratio rises at lead time two before it falls, as on the ring (Figure~\ref{fig:2d}c). So the boundary between one and two correlation lengths holds on a second geometry, in a field that is chaotic at twice the rate of the ring. Our evidence is weaker at the edges. At one correlation length the inversion is $-0.07$ against $-0.18$ on the ring. At two the rise only just excludes zero.

As the lead time grows, splits of a fixed input budget with long histories lose ground to wide windows in each field (Appendix~\ref{app:budget}). We split about 36 readings per query between window and history on newly generated data sets of the three one dimensional fields. In Burgers flow the split chosen on validation at lead time one ends with 13.5 times the error of the best split at lead time eight. The shift follows from the growth of the spatial gap with lead time that Section~\ref{sec:exact} predicts. Its direction is resolved in each field. Its cost is unresolved in KS, where reusing the choice made at lead time one raises the error by 12\%, with a 98.75\% interval of $[-1\%, +24\%]$. For a sensing system with a fixed number of readings per estimate, the allocation between neighbours and past readings should therefore follow the forecast horizon.

\section{Related work}
\label{sec:related}

Prior work on memory varies the history at a fixed spatial input and compares fitted models with each other. \citet{buitrago2025memory} show that memory reduces neural-operator error under lossy observation. A state space neural operator combines memory and space \citep{koren2026ssno}. \citet{hur2026memory} learn how much memory an operator needs and find that it falls with resolution. Mori-Zwanzig theory explains why partially observed dynamics acquire memory \citep{chorin2000optimal, ma2019memory}. Delay embedding shows that the history of noise-free observations can reconstruct a hidden state \citep{takens1981detecting}. That is the route by which history substitutes for space. Our two failure modes mark where the route fails. \citet{lozanoduran2026ladder} derives information limits for reduced models of truncated KS. \citet{wang2026knee} pick a PDE surrogate's history length by a tolerance rule like our sufficiency radius. We vary both against an exact reference.

Spatial statistics and data assimilation answer our question at a single time. In the screening effect behind nearest-neighbour Gaussian processes \citep{vecchia1988, datta2016nngp}, kriging from nearby observations nearly matches kriging from all \citep{stein2002screening, stein2011screening}. Proposition~\ref{prop:independent} is its exact, degenerate case. Ensemble Kalman filters localize each update to a radius that corrects sampling error in estimated covariances \citep{houtekamer1998enkf, hamill2001distance, hunt2007letkf}. Our sufficiency radius belongs to the exact predictor and changes with history. Emulators trained by data assimilation from sparse, noisy observations lose skill smoothly with noise and abruptly once less than half the domain is observed \citep{brajard2020combining}. \citet{tsiamis2020kalman} bound how far a Kalman filter learned for an unknown, partially observed linear system falls from the optimal one. We measure that kind of shortfall for learned forecasters. \citet{covert2020sage} use the same risk difference for feature importance. Effective receptive fields \citep{luo2016erf} describe one trained model, the spatial gap the observation channel.

Local forecasting of spatiotemporal chaos combines windows with memory, without asking how the share of the error that memory removes changes with lead time. Local states of nearby sites and time delays go back to \citet{parlitz2000local}. Noisy data favour long delays with small neighbourhoods \citep{isensee2020local}. For KS, parallel reservoir computers forecast from local windows with memory \citep{pathak2018parallel, vlachas2020backprop}, coupled local patches without memory \citep{constanteamores2024}. \citet{fleddermann2025chaos} find that the best neighbourhood length is on the scale of the spatial correlation, in line with our boundary at one to two correlation lengths. Local causal states are the optimal local information set with history \citep{shalizi2006lightcone}. \citet{ye2024locality} match the receptive range of a local neural operator to the task. For local particle filters, \citet{rebeschini2015local} show that dense exact dependence still allows good local approximation. Recent operators learn from partially observed fields \citep{hou2026latent, morel2025diffusion}. \citet{bhardwaj2025fieldformer} argue that a longer time window cannot fully compensate for missing spatial coverage. In ecological time series, memory and nonlinearity improve forecasts most where the Lyapunov exponent is positive \citep{munch2026memory}. None of these studies varies the lead time at a fixed window.

\section{Limitations}
\label{sec:limitations}

Our nonlinear gaps rest on fitted models whose optimality we cannot verify. The linear-field calibration of Section~\ref{sec:protocol} does not carry over to chaotic fields at 5\% noise, where gaps can be a few times $10^{-4}$. Only the patch GRU and MLP test whether the patch CNN, which sums features over a seven-point footprint, captures wide-window interactions. They do so at one and two correlation lengths in KS. Wide windows therefore carry the most uncertainty. The small-window results behind the inversion rest on three agreeing architectures.

Several confounds limit how far the results generalize. Each field has one noise level, which mixes field and noise in the first failure mode. The second failure mode rests on one drift speed and one window. The inversion rests on one equation, whose damped variants vary chaos and linear instability together, at lead times of at most $0.071$ Lyapunov times. A field with chaos but no linear instability, or the reverse, would separate the two. The two dimensional test covers one chaotic equation on one domain, where the rise at two correlation lengths only just excludes zero. These confounds bound the explanation of the inversion more than its existence, which holds across arrays, architectures and geometries.

\section{Conclusion}
\label{sec:conclusion}

Under degraded observation, a locally evolving field's optimal forecast reaches distant sensors only through spatial correlation. The same exact reference calibrates learned surrogates, which reach it within 1\% on the linear field. In unstable KS, a small window's history removes a larger fraction of the error at longer lead times, the reverse of the stable fields we studied. A surrogate should size its window in correlation lengths and split its input budget for the lead time it serves. Chaos and linear amplification remain to be separated as causes of the inversion. A field with chaos but no linear instability, or the reverse, would separate them.

\section*{Acknowledgments}
We thank Richard E. Turner for supervising the undergraduate project from which this work grew.

\bibliography{references}
\bibliographystyle{tmlr}

\appendix
\section{The linear Gaussian chain}
\label{app:gaussian}

The chain is small enough for exact conditioning in double precision without regularization. It has $n = 32$ sites on a ring. The update is the nearest-neighbour average $A = (1 - a_{-} - a_{+}) I + a_{-} T + a_{+} T^\top$ with the cyclic shift matrix $T$, where $a_{-} = a_{+} = 0.1$ for the diffusive update and $a_{-} = 0.15$, $a_{+} = 0.05$ for the advective one, which drifts by $0.1$ sites per step. Process noise has standard deviation $0.01$ per step. The uncorrelated state is $P = I$ with no evolution. The correlated state starts from the covariance of a smooth random field, $P_0 = (I + 4(2I - T - T^\top))^{-1}$, propagated eight steps under the update with process noise. Its correlation therefore comes from the smooth initial covariance as well as from the dynamics. Each site returns a reading with independent noise of standard deviation $\sigma$ on a logarithmic grid of eighteen points from $0$ to $100$ field units. The target is the clean value at one site $k$ steps ahead, with $k \in \{1, 2, 4, 8\}$. The window of radius $r$ holds the sites within $r$ of the target site over the last $L \in \{1, 2, 4, 8\}$ observation times. All risks are conditional variances computed in double precision with no regularization. For every conditioning covariance we record the rank and condition number. None in the reported sweeps needed truncation. A separate sweep with sensors on every second site gave gap ratios above one at large windows, where the one-snapshot gap is below $10^{-4}$ of the target variance. We omit it from the results, because a ratio of two such small gaps carries no information.

History alone carries the optimal predictor past the stencil, even from an uncorrelated start (Table~\ref{tab:history}). The state is uncorrelated, $P = I$, at the first of the $L$ snapshots. It then evolves under the update with process noise until the last snapshot. With one snapshot the spatial gap beyond the stencil is zero to machine precision. With two or more snapshots the gap is positive. It falls by two to three orders of magnitude per site of distance. Proposition~\ref{prop:independent} therefore needs its single-snapshot premise. The extra reach that history adds decays fast with distance.

\begin{table}[ht]
\centering\footnotesize\setlength{\tabcolsep}{5pt}
\caption{Exact spatial gap beyond radius $r$, $\gap(r, L, 1)/\Var$, when the state is uncorrelated at the first of $L$ snapshots and every site is observed. Each entry is the largest value over six noise levels from $\sigma = 0.1$ to $10$, with process noise $0.01$. With one snapshot the gap is at machine precision at every radius. Entries below $10^{-13}$ are at the precision of double arithmetic.}
\label{tab:history}
\begin{tabular}{llccccc}
\toprule
Update & $L$ & $r=1$ & $r=2$ & $r=3$ & $r=5$ & $r=8$ \\
\midrule
diffusive & 1 & $<10^{-13}$ & $<10^{-13}$ & $<10^{-13}$ & $<10^{-13}$ & $<10^{-13}$ \\
 & 2 & $2.3 \times 10^{-4}$ & $7.5 \times 10^{-7}$ & $3.3 \times 10^{-9}$ & $<10^{-13}$ & $<10^{-13}$ \\
 & 4 & $4.2 \times 10^{-3}$ & $5.6 \times 10^{-5}$ & $7.8 \times 10^{-7}$ & $2.9 \times 10^{-10}$ & $<10^{-13}$ \\
 & 8 & $2.9 \times 10^{-2}$ & $1.9 \times 10^{-3}$ & $5.5 \times 10^{-5}$ & $2.5 \times 10^{-7}$ & $3.3 \times 10^{-11}$ \\
\addlinespace
advective & 1 & $<10^{-13}$ & $<10^{-13}$ & $<10^{-13}$ & $<10^{-13}$ & $<10^{-13}$ \\
 & 2 & $3.0 \times 10^{-4}$ & $1.6 \times 10^{-6}$ & $5.7 \times 10^{-9}$ & $<10^{-13}$ & $<10^{-13}$ \\
 & 4 & $6.5 \times 10^{-3}$ & $6.9 \times 10^{-5}$ & $1.5 \times 10^{-6}$ & $4.6 \times 10^{-10}$ & $<10^{-13}$ \\
 & 8 & $4.5 \times 10^{-2}$ & $3.6 \times 10^{-3}$ & $8.3 \times 10^{-5}$ & $4.3 \times 10^{-7}$ & $4.3 \times 10^{-11}$ \\
\bottomrule
\end{tabular}
\end{table}

\section{Systems and data sets}
\label{app:systems}

Each data set is generated once, checked before use and fixed by content hashes. Table~\ref{tab:systems} lists the three fields and Table~\ref{tab:damped} the damped KS family. We generate each data set once, fix it with content hashes and split it by trajectory into $128$ training, $32$ validation and $100$ test trajectories of $65$ snapshots. For KS we use time step $0.05$ and keep every second step, after a burn in of $10{,}000$ steps, or $500$ time units, from a random smooth initial field. Before use, the KS data set passed resolution and stationarity checks. Burgers flow uses time step $0.0025$, keeps every tenth step and has no burn in. It decays by design and passed a resolution check. Its initial field has eight Fourier modes with independent Gaussian coefficients whose standard deviation falls as the mode number to the power $-1.5$, scaled to unit root mean square. We sample the linear field from its exact stationary law and advance it with its exact propagator every $0.1$ time units. Its forcing has spectral variance proportional to $\exp(-(m/32)^2)$ at mode $m$, with unit amplitude and relaxation $\gamma = 0.05$. For each information set, the exact predictor is the conditional mean given every reading in the window over the $L$ snapshots. We compute it from the stationary covariance and the exact propagator. We apply it to the same test windows, noise realizations and queries as the models. The damped KS fields use the KS time step and cadence and a burn in of $20{,}000$ steps. Their additive forcing has amplitude $0.044$ on modes one to sixteen, drawn afresh at every step and added after it.

Every model of a field sees the same targets and the same noise realization. We normalize every field by one mean and one standard deviation computed over its whole training split. Sensor noise has a standard deviation equal to the stated fraction of the training root mean square. We draw it once per trajectory with a fixed seed. Every model of a field therefore sees the same noise realization. A training example is one trajectory and one forecast time. Forecast times run from the eighth snapshot to the last one that leaves room for the lead time. Every history length is therefore scored on the same targets. This range gives $57$ windows per trajectory at lead time one and $50$ at lead time eight. Each example predicts all 256 grid points. A batch holds 16 examples. On the training split, the correlation length is the first grid separation at which the spatial autocorrelation falls below $1/e$ (Figure~\ref{fig:fieldstats}a). The unresolved fraction is the fraction of the field's spatial variance above the Nyquist mode of the array.

\begin{figure}[H]
\centering
\includegraphics[width=\textwidth]{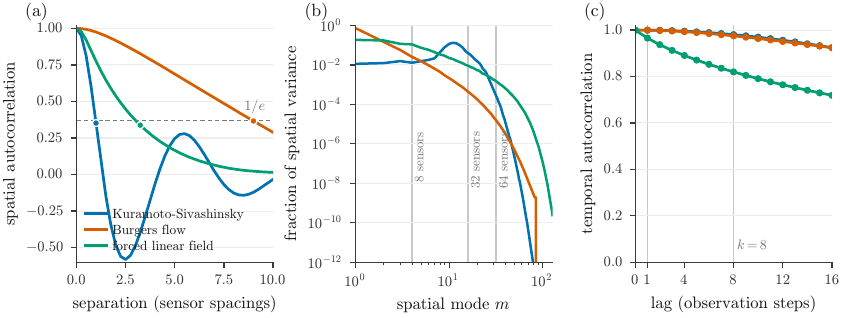}
\caption{The statistics that set the scale of each field, computed from its training split. (a) Spatial autocorrelation against separation in sensor spacings. The dot on each curve marks the first grid separation below $1/e$, which defines the correlation length. (b) Fraction of the spatial variance in each Fourier mode, with the Nyquist modes of the 64-, 32- and 8-sensor arrays. The variance to the right of a line is what that array cannot resolve. (c) Temporal autocorrelation at a fixed point against lag in observation steps, with the two lead times of the experiments.}
\label{fig:fieldstats}
\end{figure}

\begin{table}[ht]
\centering\footnotesize\setlength{\tabcolsep}{3.5pt}
\caption{The three fields, with lengths and times in the units of each equation. The observation interval is the time between snapshots. Noise is the sensor noise standard deviation as a fraction of the field's training root mean square. Where we observed a field through several arrays, entries separated by a slash follow the order of the sensor counts.}
\label{tab:systems}
\resizebox{\textwidth}{!}{\begin{tabular}{lccc}
\toprule
 & Forced linear field & Burgers flow & \ks{} \\
\midrule
Equation & $u_t + c u_x = \nu u_{xx} - \gamma u + f$ & $u_t + u u_x = \nu u_{xx}$ & $u_t + u u_x = -u_{xx} - u_{xxxx}$ \\
Domain length, grid points & $100.5$, $256$ & $2\pi$, $256$ & $100.5$, $256$ \\
Parameters & $\nu{=}1$, $\gamma{=}0.05$, $c \in \{0, 4\}$ & $\nu = 0.04$ & none \\
Forcing & white in time, 32 modes & none, decaying & none, chaotic attractor \\
Time step, observation interval & exact, $0.1$ & $0.0025$, $0.025$ & $0.05$, $0.1$ \\
Initial condition & stationary law & smooth random, no burn in & burn in of $500$ time units \\
Sensors & 64 & 64 / 8 & 64 / 256 / 32 \\
Sensor spacing & $1.57$ & $0.098$ / $0.785$ & $1.57$ / $0.39$ / $3.14$ \\
Noise & 100\% & 5\% & 5\% \\
Correlation length $\ell$ & $5.1$ & $0.88$ & $1.57$ \\
Spacing in units of $\ell$ & $0.31$ & $0.11$ / $0.89$ & $1.0$ / $0.25$ / $2.0$ \\
Variance above array Nyquist & $1.2\%$ & $0.007\%$ / $4.8\%$ & $0.06\%$ / $0$ / $13.5\%$ \\
Change per step (fraction of variance) & not measured & $0.0009$ & $0.0006$ \\
Displacement over 8 steps (spacings) & $0$ / $2$ & $1.5$ & none coherent \\
\bottomrule
\end{tabular}}
\end{table}

The two nonlinear fields are matched in time and differ in space. They change by nearly the same fraction of their variance per observation step. Over eight steps that fraction is $0.049$ for Burgers flow and $0.040$ for KS. In space, Burgers flow is smooth over nine sensor spacings between shocks. Each shock transition spans about five grid points and falls between sensors four points apart. Burgers flow also transports coherently. KS has its spectral peak at a wavelength of nine units and a correlation length equal to one sensor spacing. A lead time therefore covers comparable change in both fields. Their contrast in space and stability is what the comparisons of Section~\ref{sec:results} probe.

The exact risks of the linear field hold to machine precision. The field is an exact Gaussian chain between observation times. Its propagator matches the exponential integrator to $10^{-15}$. Its stationary covariance solves the discrete Lyapunov equation to the same precision. Any excess of a trained model over these risks is therefore fitting error.

\begin{table}[ht]
\centering\footnotesize\setlength{\tabcolsep}{5pt}
\caption{The damped and forced KS family, $u_t + u u_x + u_{xx} + u_{xxxx} + \gamma u = f$, on the KS domain with every grid point observed at 5\% noise. We chose the four damping levels from the candidates $\gamma \in \{0, 0.05, \dots, 0.40\}$ by a rule fixed before the measurements came in. We measured the exponents over 400 time units on three trajectories with common forcing. The peak linear growth rate is $1/4 - \gamma$. Growth is the median amplification of a small perturbation of one correlation length's width over eight observation steps.}
\label{tab:damped}
\begin{tabular}{ccccc}
\toprule
$\gamma$ & Lyapunov exponent & Peak linear rate & $\ell$ (grid points) & Growth over 8 steps \\
\midrule
$0$ & $+0.092 \pm 0.002$ & $+0.25$ & $1.62$ ($4.1$) & $1.13$ \\
$0.10$ & $+0.048 \pm 0.001$ & $+0.15$ & $1.68$ ($4.3$) & $1.04$ \\
$0.15$ & $+0.017 \pm 0.001$ & $+0.10$ & $1.71$ ($4.4$) & $0.98$ \\
$0.30$ & $-0.096 \pm 0.0002$ & $-0.05$ & $1.89$ ($4.8$) & $0.84$ \\
\bottomrule
\end{tabular}
\end{table}

The damped data sets pass the checks of the KS data set, with one change to the stationarity test. Each passed resolution and stationarity checks before use. Its burn in check pools the trajectories, because the single-trajectory rule used for KS rejected stationary forced fields 14\% to 29\% of the time in simulation. The pooled drift of every damped field, $0.013$ to $0.037$, is below the fixed tolerance of $0.05$ itself. No damped field therefore relies on the pooled test's allowance for noise. For reference, unforced KS on the same domain has an exponent of $0.089 \pm 0.001$ and a perturbation growth of $1.10$ over eight steps. Burgers flow has a growth of $0.98$.

\section{The two dimensional fields}
\label{app:twod}

The two dimensional fields share a correlation length and differ in dynamics, which makes the linear field an exact control for KS. Both live on the periodic square of side $8\pi$ with $64 \times 64$ grid points, stored as flattened vectors. The two dimensional KS equation in its scalar form, $u_t = -\nabla^2 u - \nabla^4 u - \tfrac12 |\nabla u|^2$, is integrated with the scheme of the one dimensional fields, with 2/3 dealiasing on each axis and time step $0.05$. The nonlinear term has a positive mean, which only drifts the spatial mean. Once we remove its zero mode, the field keeps the zero mean it starts with. Snapshots are $0.1$ time units apart after a burn in of $4000$ steps from a random field on modes one to four. The data set passed the statistical checks of the one dimensional KS field, with a spectral error of $0.28$ of the noise floor. Its variance and kurtosis errors are $1.1\%$ and $0.6\%$, with a pooled burn in drift of $0.6\%$. Over the same trajectories the largest Lyapunov exponent is $0.180 \pm 0.004$. The linear field is forced diffusion with viscosity $1$, relaxation $0.05$ and Gaussian forcing cut off at seven integer modes, sampled exactly every $0.05$ time units. Its exact stationarity residual is $10^{-18}$. With these settings both fields have a correlation length of $\pi$, or eight grid points, along each axis. Sensors on every fourth grid point leave $3 \times 10^{-5}$ of the KS variance and $2.5\%$ of the linear field's variance above their Nyquist mode. Queries sit on every fifth grid point in each direction, $169$ per snapshot, which covers every position inside a sensor cell. The two dimensional patch CNN applies three $3 \times 3$ convolutions of width $32$ to the masked square window. As the global reference we use a two dimensional Fourier neural operator with $16$ modes per axis. On the linear field a linear patch at lead time one serves as the empirical form of the exact optimum. KS has five training seeds and the linear field three. On the linear field the patch CNN comes within 0.3\% of the exact optimum of every window and history at lead time one. We chose the forcing cutoff, the sensor spacing and the linear field's cadence from exact risks computed before any data set existed. Denser sensors would have left history little to do. With sensors on every second grid point, eight snapshots removed only 11\% of the exact one-snapshot error, because averaging many noisy sensors already removes most of the noise.

\section{Allocating a fixed input budget}
\label{app:budget}

The four splits of the input budget differ by a small imbalance that cannot produce the measured shift. This experiment reuses the solver and data settings of the three one dimensional fields with a new data seed. Its 128 training, 32 validation and 100 test trajectories are therefore new. Because every grid point carries a sensor, a window of half-width $m$ grid points holds $2m+1$ sensors. The four splits are $m = 16, 8, 4, 2$ with histories $1, 2, 4, 8$, which give $33$ to $40$ readings per query. The budget therefore favours history by up to $21\%$. It does so by the same amount at every lead time. A bias that is constant across lead times cannot create a change of the best split with lead time. We trained the patch CNN at lead times $1$, $2$, $4$ and $8$ with five seeds, or three on the linear field. We fixed four predictions in advance. Writing $S$ for the split with the widest window ($m = 16$, one snapshot) and $H$ for the split with the longest history ($m = 2$, eight snapshots), the shift is $C = [\log R(H,8) - \log R(S,8)] - [\log R(H,1) - \log R(S,1)]$. It is positive in KS ($\ci{+1.60}{+1.48}{+1.71}$), Burgers flow ($\ci{+3.09}{+2.83}{+3.32}$) and the linear field ($\ci{+0.305}{+0.296}{+0.313}$, exact $+0.304$). All three intervals exclude zero. The linear value matches the exact one. The fourth prediction concerned the cost in KS of the split chosen at lead time one, which came out as predicted on validation and is unresolved on test. On the linear field the patch CNN sits within $1.6\%$ of the exact optimum on the same test windows, at every split and lead time. The split chosen on validation in KS moves from 17 sensors over two snapshots at lead time one to 33 over one at lead time eight. On the linear field the same split is best at every lead time, for the networks and for the exact predictor.

\begin{figure}[H]
\centering
\includegraphics[width=\textwidth]{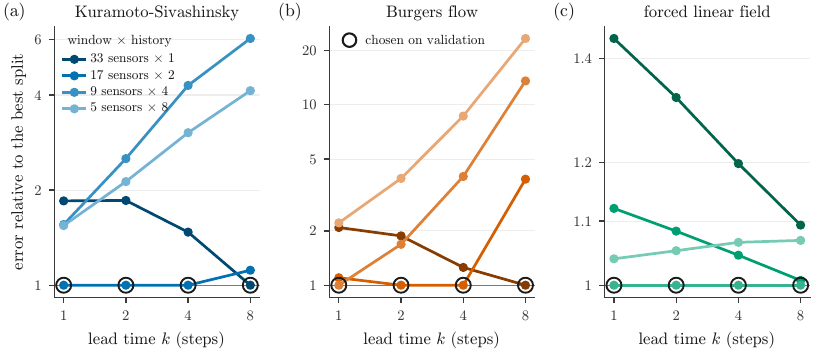}
\caption{Splitting a fixed input budget of about 36 readings per query between window and history (test split). Each line is one split, from 33 sensors over one snapshot (darkest) to five sensors over eight (lightest). The vertical axis gives its error relative to the best split at the same lead time. Rings mark the split chosen on validation. (a) In KS the chosen split moves from 17 sensors over two snapshots to 33 over one between lead times four and eight. (b) In Burgers flow the split with the widest window becomes best as the lead time grows. The split chosen at lead time one ends 13.5 times above the best. (c) On the linear field the same split is best at every lead time, as for the exact predictor. In each field the splits with long histories lose ground as the lead time grows.}
\label{fig:budget}
\end{figure}

\section{Models and training}
\label{app:models}

Patch models receive the masked window over $L$ snapshots as $2L + 1$ channels: the observed values and a sensor mask for each snapshot, plus one coordinate channel. The patch CNN applies three one dimensional convolutions of kernel size three and width $32$ with GELU activations, averages the resulting features over the window and maps them to the prediction with a linear layer. It has $6{,}561$ parameters at $L = 1$ and $7{,}905$ at $L = 8$, whatever the window. For the patch GRU, the same convolutional stack encodes each snapshot separately. A gated recurrent unit of width $32$ then reads the $L$ feature vectors, starting from a zero hidden state in every window. Its final state gives the output. The patch GRU has $12{,}897$ parameters. The patch MLP is a perceptron with two hidden layers of width $32$, with $1{,}985$ to $10{,}369$ parameters depending on the window and history. A single linear layer on the flattened window forms the linear patch model, with $16$ to $4{,}353$. Because the mask comes before any learned operation, no model reads outside its window. Our Fourier neural operator lifts the values and masks of all $L$ snapshots, $2L$ channels on the 256-point grid, with a pointwise layer. It then applies three spectral layers of width $32$ with $64$ retained modes, $396{,}513$ to $396{,}961$ real parameters depending on the history. Training uses AdamW with learning rate $10^{-3}$, weight decay $0$ and batch size $16$. A cosine schedule decays the learning rate to $10^{-5}$ over $150$ epochs, with no early stopping. We keep the checkpoint with the lowest validation error, which under this schedule is the final one for nearly every patch model. Seeds are $2684$ to $2688$. On the linear field the operator's validation error is lowest at its first epoch and rises after that. Its selected checkpoint is therefore the one-epoch model. Each run records its configuration, input hashes, source fingerprint, seed and hardware in a manifest. Every analysis first checks each stored result against its manifest. The $1{,}786$ training runs reported here took $275$ GPU-hours on a shared cluster of mixed NVIDIA GPUs. The evaluations took a small fraction of that time.

\section{Intervals and references}
\label{app:stats}

We first average each model's test error over the windows and query points of each test trajectory, which gives a matrix over training seeds and trajectories. Each of 2000 bootstrap draws resamples the seeds with replacement and, separately, the trajectories with replacement. It then recomputes the statistic from the resampled means. Because every model in a comparison uses the same drawn indices, differences and ratios stay paired. The intervals are percentile intervals, with no adjustment across experiments. With five seeds the seed dimension resamples only five values. For every lead-time comparison we therefore also recompute the point estimate with each seed left out (Table~\ref{tab:leadtime}). The intervals exclude variation across training sets and noise realizations, because each field has one fixed data set. The agreement between validation and test values is our check on selection bias.

Two properties of the models limit what they can show. The patch CNN's footprint is seven grid points. Across a wider window it therefore sums local features and cannot form products of distant readings. The patch MLP has no such restriction. Its agreement with the CNN at one and two correlation lengths (Section~\ref{sec:leadtime}) suggests that the footprint does not drive the small-window results. In both nonlinear fields the patch CNN whose window spans the whole ring is limited by optimization, with a seed spread several times that of any narrower window. For the global reference, the Fourier neural operator keeps 64 modes. A 12-mode version we tried first discarded modes that the sensors carry and produced negative apparent gaps.

\section{Planned comparisons and their outcomes}
\label{app:plans}

Each experiment has a dated design document and a written analysis plan. The plan names the comparisons that decide the experiment's outcome, the direction we predicted for each and the interval level. We recorded a content hash of each plan after the validation analysis and before any test evaluation. Table~\ref{tab:decisions} lists all fifty planned comparisons of the seven experiments with their outcomes. Forty came out as predicted, eight did not and two are unresolved. Forty-nine of the fifty have the same outcome on validation as on test.

We count the failed predictions as part of the result. The most important one is our expectation that in KS the history ratio would rise with lead time at one sensor spacing. That failure is the inversion of Section~\ref{sec:leadtime}. We also expected eight snapshots of a single sensor to close less than half of that sensor's spatial gap. This held in KS. Burgers flow and both linear variants carried the same prediction over unchanged. There it failed. This failure showed that the first failure mode depends on the field (Section~\ref{sec:context}). On the diffusive linear field the history ratio fell with the window, against our prediction. Our written motivation had quoted the exact gap ratios, which do fall there. Still, the comparison we fixed in advance used the history ratio. We therefore report it as a failure. For both linear variants we also asked the fitted gap against the global operator to lie within 25\% of the exact gap at every window up to $3.2$. This fails where the exact gap is below $10^{-5}$, because the operator's own excess of $0.004$ to $0.006$ is larger than the gap there. The linear-field plan also named the lead-time comparison at $r = 1.6$, which its analysis code did not compute. We computed it afterwards with the same estimator on both splits.

Two comparisons are unresolved. In the damped KS family the non-chaotic field shows no lead-time difference, for the reason Section~\ref{sec:leadtime} gives. The other is the cost in KS of reusing at lead time eight the split of a fixed input budget chosen at lead time one. It came out as predicted on validation and is unresolved on test, the one planned comparison that differs between the splits. One comparison outside the planned set also changed. Burgers flow observed through eight sensors at $1.8$ correlation lengths excluded zero on validation. On test its interval includes zero, although its point estimate is positive on both splits.

\begin{table}[ht]
\centering\footnotesize\setlength{\tabcolsep}{3pt}
\caption{The planned comparisons of the seven experiments on the test split, with the direction we predicted before any test evaluation. Intervals are at 99.17\%, or at 98.75\% for the damped KS family and the input budget. 40 of 50 comparisons came out as predicted. 49 of 50 agree between validation and test. $\rho(r)$ is the history ratio at lead time one and window $r$.}
\label{tab:decisions}
\resizebox{\textwidth}{!}{\begin{tabular}{lllll}
\toprule
Experiment & Comparison & Predicted & Test value & Outcome \\
\midrule
KS, 64 and 32 sensors & 64 sensors, $\rho(0.79) - \rho(1.6)$ & $> 0$ & $+0.307$ $[+0.292, +0.322]$ & as predicted \\
 & 64 sensors, $\rho(1.6) - \rho(3.2)$ & $> 0$ & $+0.216$ $[+0.201, +0.231]$ & as predicted \\
 & 32 sensors, $\rho(1.58) - \rho(3.2)$ & $> 0$ & $+0.111$ $[+0.101, +0.122]$ & as predicted \\
 & 32 sensors, $\rho(3.2) - \rho(6.3)$ & $> 0$ & $+0.083$ $[+0.067, +0.097]$ & as predicted \\
 & one-sensor closed fraction at $r=0.79$ & $< 0.5$ & $+0.195$ $[+0.183, +0.208]$ & as predicted \\
 & lead-time difference at $r=1.6$ & $> 0$ & $-0.177$ $[-0.186, -0.167]$ & not as predicted \\
\addlinespace
Burgers flow, 64 sensors & $\rho(0.05) - \rho(0.1)$ & $> 0$ & $+0.082$ $[+0.063, +0.102]$ & as predicted \\
 & $\rho(0.1) - \rho(0.2)$ & $> 0$ & $+0.158$ $[+0.142, +0.175]$ & as predicted \\
 & one-sensor closed fraction at $r=0.05$ & $< 0.5$ & $+0.752$ $[+0.678, +0.938]$ & not as predicted \\
 & lead-time difference at $r=0.05$ & $> 0$ & $+0.061$ $[+0.039, +0.082]$ & as predicted \\
 & lead-time difference at $r=0.1$ & $> 0$ & $+0.081$ $[+0.055, +0.112]$ & as predicted \\
 & closed fraction, lead 8 minus lead 1, $r=0.05$ & $< 0$ & $-0.252$ $[-0.516, -0.165]$ & as predicted \\
\addlinespace
Linear field, diffusive & linear patch excess over the exact optimum & upper bound $\le 10\%$ & worst upper bound $2.9\%$ & as predicted \\
 & CNN excess minus linear patch excess & upper bound $\le 0.05$ & $+0.0023$ $[+0.0010, +0.0038]$ & as predicted \\
 & $\rho(0.79) - \rho(1.6)$ & $> 0$ & $-0.018$ $[-0.022, -0.014]$ & not as predicted \\
 & $\rho(1.6) - \rho(3.2)$, coded in addition & $> 0$ & $-0.052$ $[-0.056, -0.048]$ & not as predicted \\
 & one-sensor closed fraction, carried over from KS & $< 0.5$ & $+1.499$ $[+1.456, +1.546]$ & not as predicted \\
 & lead-time difference at $r=0.79$ & $> 0$ & $+0.177$ $[+0.171, +0.183]$ & as predicted \\
 & lead-time difference at $r=1.6$, computed afterwards & $> 0$ & $+0.192$ $[+0.186, +0.199]$ & as predicted \\
 & fitted gap against the operator within $\pm25\%$ of the exact gap & every window & fails at $r = 3.2$ & not as predicted \\
\addlinespace
Linear field, advective & linear patch excess over the exact optimum & upper bound $\le 10\%$ & worst upper bound $3.3\%$ & as predicted \\
 & CNN excess minus linear patch excess & upper bound $\le 0.05$ & $+0.0029$ $[+0.0004, +0.0057]$ & as predicted \\
 & $\rho(0.79) - \rho(1.6)$ & $> 0$ & $+0.055$ $[+0.050, +0.060]$ & as predicted \\
 & $\rho(1.6) - \rho(3.2)$, coded in addition & $> 0$ & $+0.053$ $[+0.047, +0.059]$ & as predicted \\
 & one-sensor closed fraction, carried over from KS & $< 0.5$ & $+0.919$ $[+0.905, +0.934]$ & not as predicted \\
 & lead-time difference at $r=0.79$ & $> 0$ & $+0.310$ $[+0.302, +0.318]$ & as predicted \\
 & lead-time difference at $r=1.6$, computed afterwards & $> 0$ & $+0.353$ $[+0.347, +0.359]$ & as predicted \\
 & fitted gap against the operator within $\pm25\%$ of the exact gap & every window & fails at $r = 3.2$ & not as predicted \\
\addlinespace
Instability or resolution & lead-time difference, KS, 256 sensors, 0.25 CL & $< 0$ & $-0.275$ $[-0.281, -0.270]$ & as predicted \\
 & lead-time difference, KS, 256 sensors, 1 CL & $< 0$ & $-0.181$ $[-0.190, -0.172]$ & as predicted \\
 & lead-time difference, Burgers, 8 sensors, 0.45 CL & $> 0$ & $+0.059$ $[+0.042, +0.078]$ & as predicted \\
 & lead-time difference, Burgers, 8 sensors, 0.89 CL & $> 0$ & $+0.024$ $[+0.003, +0.045]$ & as predicted \\
 & lead-time difference, KS, GRU, 1 CL & $< 0$ & $-0.205$ $[-0.215, -0.195]$ & as predicted \\
 & lead-time difference, KS, MLP, 1 CL & $< 0$ & $-0.203$ $[-0.212, -0.193]$ & as predicted \\
\addlinespace
Damped KS family & $D$, $\gamma = 0$, 0.25 CL & $< 0$ & $-0.366$ $[-0.376, -0.356]$ & as predicted \\
 & $D$, $\gamma = 0.30$, 0.25 CL & $> 0$ & $+0.000$ $[-0.001, +0.001]$ & unresolved \\
 & $D$, $\gamma = 0.30$ minus $\gamma = 0$, 0.25 CL & $> 0$ & $+0.367$ $[+0.357, +0.377]$ & as predicted \\
 & $D$, $\gamma = 0.30$ minus $\gamma = 0$, 1 CL & $> 0$ & $+0.139$ $[+0.126, +0.152]$ & as predicted \\
 & rank of $D$ at 0.25 CL against the exponent & Spearman $-1$ & $-1$, order in all draws & as predicted \\
\addlinespace
Two dimensional fields & lead-time difference, 2D KS, 0.5 CL & $< 0$ & $-0.357$ $[-0.371, -0.345]$ & as predicted \\
 & lead-time difference, 2D KS, 1 CL & $< 0$ & $-0.073$ $[-0.104, -0.046]$ & as predicted \\
 & lead-time difference, 2D KS, 2 CL & $> 0$ & $+0.060$ $[+0.006, +0.105]$ & as predicted \\
 & lead-time difference, 2D linear, 0.5 CL & $> 0$ & $+0.178$ $[+0.174, +0.182]$ & as predicted \\
 & 2D KS, $\rho(0.5\,\mathrm{CL}) - \rho(1\,\mathrm{CL})$ & $> 0$ & $+0.266$ $[+0.245, +0.288]$ & as predicted \\
 & 2D KS, $\rho(1\,\mathrm{CL}) - \rho(2\,\mathrm{CL})$ & $> 0$ & $+0.090$ $[+0.013, +0.144]$ & as predicted \\
 & 2D CNN excess over the exact optimum & upper bound $\le 3\%$ & worst upper bound $0.7\%$ & as predicted \\
\addlinespace
Fixed input budget & KS, shift of the best split toward wider windows, $C$ & $> 0$ & $+1.597$ $[+1.475, +1.708]$ & as predicted \\
 & Burgers, shift of the best split toward wider windows, $C$ & $> 0$ & $+3.089$ $[+2.826, +3.323]$ & as predicted \\
 & linear, shift of the best split toward wider windows, $C$ & $> 0$ & $+0.305$ $[+0.296, +0.313]$ & as predicted \\
 & KS, cost at lead 8 of the split chosen at lead 1 & $> 0$ & $+0.115$ $[-0.011, +0.239]$ & unresolved (differs on validation) \\
\bottomrule
\end{tabular}}
\end{table}

\section{Calibration against the exact optimum}
\label{app:calib}

\begin{figure}[H]
\centering
\includegraphics[width=\textwidth]{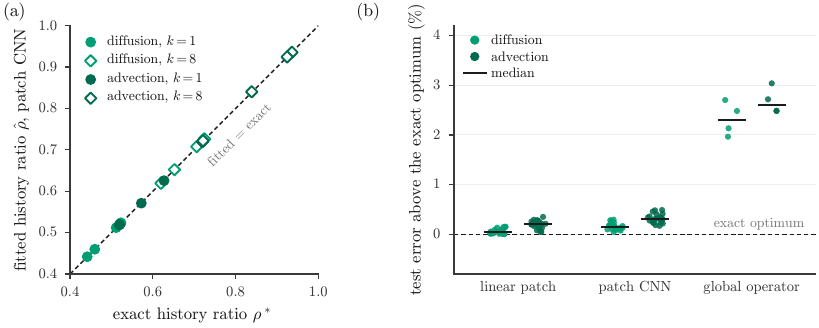}
\caption{Calibration on the forced linear field, where the exact risk of every information set is known. (a) The history ratio of the trained patch CNN against the exact ratio of the same information set, for both variants of the field, both lead times and every window from one to nine sensors. Every point lies within $0.002$ of the identity line. (b) Test error of each model family relative to the exact optimum of its own information set, one point per window and history at lead time one, with bars at the median. At lead time one the linear patch and the patch CNN sit within 0.5\% of the optimum. Over all lead times they sit within 0.65\%. The global Fourier neural operator sits 2\% to 3\% above the optimum, which we attribute to its fit to the sensor noise.}
\label{fig:calibration}
\end{figure}

\section{The error surface of the two nonlinear fields}
\label{app:surface}

\begin{figure}[H]
\centering
\includegraphics[width=\textwidth]{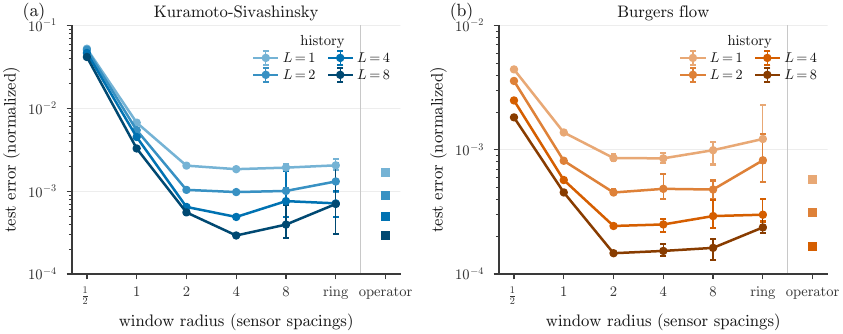}
\caption{Test error against window and history at lead time one, with 64 sensors. Each curve is one history $L$. Points are means over five training seeds. Error bars span the seeds. The rightmost square in each panel is the global Fourier neural operator reading every sensor with the same history. (a) In KS the error falls steeply up to a window of four or five sensors and is nearly flat beyond it. The window spanning the whole ring gives the noisiest result, because a 256-point additive patch is hard to train. (b) Burgers flow has the same shape, with errors an order of magnitude lower.}
\label{fig:risk}
\end{figure}

\section{Lead-time comparisons and gap ratios}
\label{app:contrasts}

The lead-time results hold at the windows we tested and under refits that leave out one training seed. Table~\ref{tab:decisive} gives the six planned comparisons of the experiment of Section~\ref{sec:leadtime} that separates instability from resolution. Table~\ref{tab:leadtime} gives the history ratios, their lead-time difference and the error reduction from history for every window of every field, including windows without an advance prediction. Its last column is the range of the difference when each training seed is left out in turn. No sign changes under that check. The error reduction $B$ grows with lead time at the small windows of both nonlinear fields and of the three unstable damped fields. It shrinks on both linear fields and on the eight-sensor Burgers array. The gap ratio follows the exact ordering only in part. It can disagree with the history ratio. On the noisy linear field the history ratio rises from $0.44$ to $0.52$ as the window widens. Over the same windows the exact gap ratio falls from $0.28$ to $0.09$. Table~\ref{tab:gapratio} gives it against the global operator, the substitution measure of the exact analysis, for every array with an operator reference. At lead time one it falls with the window with 32 KS sensors and in Burgers flow. With 64 KS sensors it falls from $0.82$ to $0.59$ but then rises to $\ci{0.72}{0.68}{0.77}$. At lead time eight in KS it rises at every step, with $0.32$, $0.54$ and $0.91$ at one, two and four correlation lengths.

\begin{table}[H]
\centering\footnotesize\setlength{\tabcolsep}{4pt}
\caption{The six planned lead-time comparisons of the experiment that separates instability from resolution (test split, 99.17\% intervals). A negative difference with its interval below zero means that eight snapshots remove a larger fraction of the one-snapshot error at lead time eight. Windows are given in correlation lengths (CL) and in the sensors a query sees. The last column gives the error reduction from history, $B$, at lead times one and eight. The outcomes are the same on validation.}
\label{tab:decisive}
\resizebox{\textwidth}{!}{\begin{tabular}{llcclc}
\toprule
System and model & Window & $\rho$, $k{=}1 \to 8$ & Difference [99.17\%] & Outcome & $B \times 10^{3}$, $k{=}1 \to 8$ \\
\midrule
KS, 256 sensors, CNN & 0.25 CL, 3 sensors & $0.41 \to 0.14$ & $\ci{-0.275}{-0.281}{-0.270}$ & $\rho$ falls & $0.83 \to 28.2$ \\
KS, 256 sensors, CNN & 1 CL, 9 sensors & $0.40 \to 0.22$ & $\ci{-0.181}{-0.190}{-0.172}$ & $\rho$ falls & $0.44 \to 6.41$ \\
KS, 64 sensors, GRU & 1 CL, 2 to 3 sensors & $0.53 \to 0.32$ & $\ci{-0.205}{-0.215}{-0.195}$ & $\rho$ falls & $3.17 \to 17.4$ \\
KS, 64 sensors, MLP & 1 CL, 2 to 3 sensors & $0.52 \to 0.32$ & $\ci{-0.203}{-0.212}{-0.193}$ & $\rho$ falls & $3.22 \to 17.6$ \\
Burgers, 8 sensors, CNN & 0.45 CL, 1 sensor & $0.69 \to 0.75$ & $\ci{+0.059}{+0.042}{+0.078}$ & $\rho$ rises & $48.9 \to 41.9$ \\
Burgers, 8 sensors, CNN & 0.89 CL, 2 to 3 sensors & $0.69 \to 0.71$ & $\ci{+0.024}{+0.003}{+0.045}$ & $\rho$ rises & $31.8 \to 28.0$ \\
\bottomrule
\end{tabular}}
\end{table}

\begin{table}[ht]
\centering\footnotesize\setlength{\tabcolsep}{3pt}
\caption{Lead-time comparisons on the test split. The history ratios at lead times one and eight are $\rho_1$ and $\rho_8$. $D = \rho_8 - \rho_1$ is their difference, with its 99.17\% interval. $B_1$ and $B_8$ are the error reductions from history. The last column gives the range of $D$ over the five fits that each leave out one training seed. Windows are in correlation lengths (CL) of each field. A full-ring window holds every sensor.}
\label{tab:leadtime}
\resizebox{\textwidth}{!}{\begin{tabular}{llcccccc}
\toprule
System & Window & $\rho_1$ & $\rho_8$ & $D$ [99.17\%] & $B_1 \times 10^3$ & $B_8 \times 10^3$ & $D$, one seed left out \\
\midrule
KS, 64 sensors, CNN & 1 CL & 0.488 & 0.311 & $-0.177$ $[-0.186, -0.167]$ & 3.45 & 17.9 & $[-0.177, -0.175]$ \\
 & 2 CL & 0.272 & 0.449 & $+0.177$ $[+0.164, +0.190]$ & 1.48 & 4.12 & $[+0.173, +0.180]$ \\
 & 4 CL & 0.158 & 0.355 & $+0.197$ $[+0.166, +0.247]$ & 1.56 & 1.66 & $[+0.182, +0.207]$ \\
 & full ring & 0.341 & 0.510 & $+0.169$ $[-0.226, +0.453]$ & 1.35 & 1.61 & $[+0.099, +0.298]$ \\
\addlinespace
KS, 256 sensors, CNN & 0.25 CL & 0.412 & 0.137 & $-0.275$ $[-0.281, -0.270]$ & 0.825 & 28.2 & $[-0.277, -0.274]$ \\
 & 0.5 CL & 0.354 & 0.166 & $-0.187$ $[-0.192, -0.183]$ & 0.748 & 14 & $[-0.188, -0.187]$ \\
 & 1 CL & 0.398 & 0.217 & $-0.181$ $[-0.190, -0.172]$ & 0.439 & 6.41 & $[-0.182, -0.180]$ \\
 & 2 CL & 0.145 & 0.342 & $+0.198$ $[+0.190, +0.206]$ & 0.461 & 0.772 & $[+0.196, +0.199]$ \\
 & full ring & 0.252 & 0.748 & $+0.496$ $[+0.107, +1.080]$ & 0.402 & 0.205 & $[+0.258, +0.652]$ \\
\addlinespace
KS, 64 sensors, GRU & 1 CL & 0.528 & 0.323 & $-0.205$ $[-0.215, -0.195]$ & 3.17 & 17.4 & $[-0.206, -0.205]$ \\
 & 2 CL & 0.257 & 0.452 & $+0.196$ $[+0.185, +0.207]$ & 1.5 & 3.56 & $[+0.193, +0.197]$ \\
\addlinespace
KS, 64 sensors, MLP & 1 CL & 0.520 & 0.318 & $-0.203$ $[-0.212, -0.193]$ & 3.22 & 17.6 & $[-0.203, -0.202]$ \\
 & 2 CL & 0.220 & 0.393 & $+0.172$ $[+0.162, +0.182]$ & 1.56 & 3.7 & $[+0.171, +0.174]$ \\
\addlinespace
Burgers, 64 sensors & 0.057 CL & 0.411 & 0.472 & $+0.061$ $[+0.039, +0.082]$ & 2.62 & 17.5 & $[+0.060, +0.061]$ \\
 & 0.11 CL & 0.329 & 0.410 & $+0.081$ $[+0.055, +0.112]$ & 0.926 & 10.7 & $[+0.080, +0.082]$ \\
 & 0.23 CL & 0.171 & 0.534 & $+0.363$ $[+0.334, +0.393]$ & 0.712 & 1.89 & $[+0.359, +0.367]$ \\
 & 0.45 CL & 0.180 & 0.379 & $+0.199$ $[+0.121, +0.265]$ & 0.701 & 0.786 & $[+0.186, +0.225]$ \\
 & full ring & 0.194 & 0.630 & $+0.436$ $[+0.228, +0.821]$ & 0.986 & 0.705 & $[+0.349, +0.524]$ \\
\addlinespace
Burgers, 8 sensors & 0.45 CL & 0.689 & 0.748 & $+0.059$ $[+0.042, +0.078]$ & 48.9 & 41.9 & $[+0.058, +0.061]$ \\
 & 0.89 CL & 0.689 & 0.713 & $+0.024$ $[+0.003, +0.045]$ & 31.8 & 28 & $[+0.020, +0.028]$ \\
 & 1.8 CL & 0.713 & 0.738 & $+0.025$ $[-0.009, +0.052]$ & 28.3 & 24.5 & $[+0.019, +0.033]$ \\
 & full ring & 0.739 & 0.760 & $+0.022$ $[+0.002, +0.040]$ & 24.3 & 21.2 & $[+0.019, +0.027]$ \\
\addlinespace
Linear, diffusive & 0.15 CL & 0.442 & 0.619 & $+0.177$ $[+0.171, +0.183]$ & 290 & 242 & $[+0.177, +0.178]$ \\
 & 0.31 CL & 0.460 & 0.652 & $+0.192$ $[+0.186, +0.199]$ & 220 & 183 & $[+0.192, +0.192]$ \\
 & 0.63 CL & 0.512 & 0.708 & $+0.196$ $[+0.190, +0.202]$ & 174 & 134 & $[+0.196, +0.196]$ \\
 & 1.2 CL & 0.523 & 0.726 & $+0.203$ $[+0.197, +0.209]$ & 166 & 122 & $[+0.203, +0.203]$ \\
 & full ring & 0.523 & 0.726 & $+0.203$ $[+0.197, +0.209]$ & 166 & 122 & $[+0.203, +0.203]$ \\
\addlinespace
Linear, advective & 0.15 CL & 0.626 & 0.936 & $+0.310$ $[+0.302, +0.318]$ & 204 & 57.4 & $[+0.309, +0.310]$ \\
 & 0.31 CL & 0.571 & 0.925 & $+0.353$ $[+0.347, +0.359]$ & 180 & 60.3 & $[+0.353, +0.353]$ \\
 & 0.63 CL & 0.518 & 0.840 & $+0.322$ $[+0.315, +0.329]$ & 174 & 98.3 & $[+0.321, +0.323]$ \\
 & 1.2 CL & 0.519 & 0.720 & $+0.200$ $[+0.194, +0.206]$ & 170 & 129 & $[+0.200, +0.200]$ \\
 & full ring & 0.519 & 0.722 & $+0.203$ $[+0.196, +0.208]$ & 170 & 126 & $[+0.202, +0.203]$ \\
\addlinespace
Damped KS, $\gamma = 0$ & 0.25 CL & 0.798 & 0.431 & $-0.366$ $[-0.377, -0.356]$ & 0.719 & 43 & $[-0.367, -0.366]$ \\
 & 0.5 CL & 0.821 & 0.633 & $-0.188$ $[-0.201, -0.175]$ & 0.559 & 15.7 & $[-0.189, -0.188]$ \\
 & 1 CL & 0.906 & 0.768 & $-0.137$ $[-0.150, -0.124]$ & 0.237 & 6.68 & $[-0.138, -0.137]$ \\
 & 2 CL & 0.934 & 0.973 & $+0.039$ $[+0.030, +0.048]$ & 0.151 & 0.486 & $[+0.037, +0.042]$ \\
\addlinespace
Damped KS, $\gamma = 0.10$ & 0.25 CL & 0.950 & 0.814 & $-0.136$ $[-0.154, -0.119]$ & 0.244 & 9.61 & $[-0.136, -0.136]$ \\
 & 0.5 CL & 0.921 & 0.906 & $-0.015$ $[-0.024, -0.007]$ & 0.374 & 4.03 & $[-0.015, -0.015]$ \\
 & 1 CL & 0.959 & 0.933 & $-0.025$ $[-0.032, -0.020]$ & 0.179 & 2.52 & $[-0.026, -0.024]$ \\
 & 2 CL & 0.970 & 0.993 & $+0.024$ $[+0.019, +0.031]$ & 0.125 & 0.204 & $[+0.022, +0.025]$ \\
\addlinespace
Damped KS, $\gamma = 0.15$ & 0.25 CL & 0.977 & 0.936 & $-0.041$ $[-0.048, -0.033]$ & 0.168 & 3.97 & $[-0.041, -0.041]$ \\
 & 0.5 CL & 0.949 & 0.966 & $+0.018$ $[+0.015, +0.021]$ & 0.364 & 1.94 & $[+0.017, +0.018]$ \\
 & 1 CL & 0.976 & 0.969 & $-0.007$ $[-0.010, -0.003]$ & 0.159 & 1.65 & $[-0.007, -0.007]$ \\
 & 2 CL & 0.982 & 1.000 & $+0.018$ $[+0.015, +0.023]$ & 0.119 & 0.0114 & $[+0.017, +0.019]$ \\
\addlinespace
Damped KS, $\gamma = 0.30$ & 0.25 CL & 0.996 & 0.997 & $+0.000$ $[-0.001, +0.002]$ & 0.103 & 0.652 & $[+0.000, +0.000]$ \\
 & 0.5 CL & 0.988 & 0.997 & $+0.008$ $[+0.007, +0.010]$ & 0.343 & 0.674 & $[+0.008, +0.009]$ \\
 & 1 CL & 0.997 & 0.999 & $+0.002$ $[+0.000, +0.004]$ & 0.0784 & 0.129 & $[+0.002, +0.003]$ \\
 & 2 CL & 0.998 & 1.003 & $+0.005$ $[+0.003, +0.008]$ & 0.0633 & -0.629 & $[+0.004, +0.006]$ \\
\bottomrule
\end{tabular}}
\end{table}

\begin{table}[ht]
\centering\footnotesize\setlength{\tabcolsep}{4pt}
\caption{Gap ratio $\hat g(r,k) = \hat\gap(r,8,k)/\hat\gap(r,1,k)$ against the global operator, the measure of substitution in the exact analysis, with 95\% intervals (test split). We added this analysis after the test evaluation. We omit rows whose one-snapshot gap is below $2 \times 10^{-4}$, because a ratio of two such small gaps carries no information.}
\label{tab:gapratio}
\begin{tabular}{llcc}
\toprule
Array & Window & lead time 1 & lead time 8 \\
\midrule
KS, 64 sensors & 0.5 CL & $0.81$ $[0.81, 0.82]$ &  \\
 & 1 CL & $0.59$ $[0.58, 0.61]$ & $0.32$ $[0.31, 0.33]$ \\
 & 2 CL & $0.72$ $[0.68, 0.77]$ & $0.54$ $[0.52, 0.55]$ \\
 & 4 CL &  & $0.91$ $[0.67, 1.37]$ \\
 & 8 CL & $0.40$ $[-0.02, 1.10]$ &  \\
\addlinespace
KS, 32 sensors & 1 CL & $0.79$ $[0.76, 0.81]$ &  \\
 & 2 CL & $0.56$ $[0.51, 0.61]$ &  \\
 & 4 CL & $0.19$ $[0.08, 0.29]$ &  \\
 & 8 CL & $0.14$ $[-0.01, 0.27]$ &  \\
\addlinespace
KS, 256 sensors & 0.25 CL & $0.53$ $[0.52, 0.54]$ & $0.14$ $[0.13, 0.14]$ \\
 & 0.5 CL & $0.47$ $[0.46, 0.48]$ & $0.16$ $[0.16, 0.17]$ \\
 & 1 CL & $0.78$ $[0.76, 0.80]$ & $0.22$ $[0.21, 0.22]$ \\
 & 2 CL &  & $0.46$ $[0.44, 0.48]$ \\
\addlinespace
Burgers, 64 sensors & 0.057 CL & $0.45$ $[0.43, 0.47]$ & $0.48$ $[0.46, 0.49]$ \\
 & 0.11 CL & $0.46$ $[0.44, 0.48]$ & $0.42$ $[0.39, 0.44]$ \\
 & 0.23 CL & $0.21$ $[0.17, 0.26]$ & $0.58$ $[0.56, 0.61]$ \\
 & 0.45 CL & $0.24$ $[0.17, 0.36]$ & $0.51$ $[0.43, 0.59]$ \\
 & 0.89 CL & $0.18$ $[0.11, 0.34]$ &  \\
\addlinespace
Burgers, 8 sensors & 0.45 CL & $0.72$ $[0.68, 0.76]$ & $0.77$ $[0.74, 0.80]$ \\
 & 0.89 CL & $0.75$ $[0.70, 0.80]$ & $0.74$ $[0.70, 0.78]$ \\
 & 1.8 CL & $0.79$ $[0.74, 0.84]$ & $0.78$ $[0.73, 0.82]$ \\
\bottomrule
\end{tabular}
\end{table}

\section{Sufficiency at relative tolerances}
\label{app:suff}

A relative tolerance gives a more informative sufficiency radius than an absolute one. An absolute tolerance of $10^{-3}$ exceeds the whole KS error at eight snapshots. At long history the absolute sufficiency radius is therefore nearly vacuous. Table~\ref{tab:suff} gives each window's gap relative to the global operator's error. In KS the eight-or-nine-sensor window sits 10\% above the operator at one and two snapshots and matches it from four on. It therefore suffices at a 5\% relative tolerance only from four snapshots. In Burgers flow no window gets within 40\% of the operator at any history, although the absolute gaps are below $3 \times 10^{-4}$. In Burgers flow our data cannot tell whether this gap comes from information beyond the window or from the operator's more accurate treatment of shocks.

On coarse arrays the local models fall well short of the operator at one snapshot. History closes most of that shortfall. With 32 KS sensors the shortfall of the local models, $\ci{0.054}{0.046}{0.063}$ at one snapshot, drops to $\ci{0.008}{-0.003}{0.019}$ at eight. Burgers flow observed by eight sensors shows a similar shortfall. Part of these shortfalls on coarse arrays may come from the difficulty the local models have in reconstructing the field between distant sensors.

\begin{table}[ht]
\centering\footnotesize\setlength{\tabcolsep}{4pt}
\caption{Relative spatial gap of each window against the global operator, $\hat\gap(r,L,1)/\hat R_{\mathrm{op}}(L,1)$, at lead time one (test split, point estimates). The last two columns give the smallest window whose upper 95\% bound is at most 5\% or 10\%. An entry of none means that no window we trained meets the tolerance.}
\label{tab:suff}
\begin{tabular}{llccccccc}
\toprule
KS, 64 sensors & $L$ & 0.5 CL & 1 CL & 2 CL & 4 CL & 8 CL & 5\% & 10\% \\
\midrule
 & 1 & 30.39 & 3.01 & 0.21 & 0.10 & 0.15 & none & none \\
 & 2 & 56.73 & 5.15 & 0.17 & 0.10 & 0.13 & none & none \\
 & 4 & 94.49 & 8.22 & 0.32 & -0.00 & 0.55 & 4 CL & 4 CL \\
 & 8 & 141.01 & 10.15 & 0.88 & -0.01 & 0.34 & 4 CL & 4 CL \\
\addlinespace
Burgers, 64 sensors & $L$ & 0.057 CL & 0.11 CL & 0.23 CL & 0.45 CL & 0.89 CL & 5\% & 10\% \\
\midrule
 & 1 & 6.72 & 1.40 & 0.49 & 0.48 & 0.73 & none & none \\
 & 2 & 10.51 & 1.62 & 0.45 & 0.56 & 0.54 & none & none \\
 & 4 & 14.02 & 2.43 & 0.46 & 0.51 & 0.76 & none & none \\
 & 8 & 20.02 & 4.23 & 0.69 & 0.77 & 0.87 & none & none \\
\bottomrule
\end{tabular}
\end{table}

\section{Reproducibility}
\label{app:repro}

All configurations are strict specifications. All data sets and results carry manifests with content hashes. The analyses read only results that match their recorded hashes. Each experiment has a dated design document and a written analysis plan. We recorded the hash of each plan before the single test evaluation of that experiment. From the analysis tables, one script builds the figures in this paper and records the hash of every input. A second script writes the appendix tables. The repository at \url{https://anonymous.4open.science/r/memory-space-forecasting-5633} holds the code, configurations, analysis plans and analysis tables.

\end{document}